\documentclass[final,12pt]{main} % Include author names

\title[Oracle-Efficient Hybrid Online Learning]{Oracle-Efficient Hybrid Online Learning with Unknown Distribution}
\usepackage{times}
\usepackage{graphicx}
\usepackage{xcolor}
\usepackage{tikz}
\usepackage[bottom]{footmisc}

\usetikzlibrary{decorations.pathreplacing, arrows.meta, calc}

\coltauthor{%
 \Name{Changlong Wu} \Email{wuchangl@hawaii.edu}\\
 \addr CSoI, Purdue University
 \AND
 \Name{Jin Sima} \Email{jsima@illinois.edu}\\
 \addr ECE Department, UIUC
 \AND
 \Name{Wojciech Szpankowski} \Email{szpan@purdue.edu}\\
 \addr CSoI, Purdue University
}

\newcommand{\x}{\mathbf{x}}
\newcommand{\z}{\mathbf{z}}

\newtheorem{fact}{Fact}

\allowdisplaybreaks

\begin{document}

\maketitle

\begin{abstract}%
  We study the problem of \emph{oracle-efficient} hybrid online learning when the features are generated by an \emph{unknown} i.i.d. process and the labels are generated adversarially. Assuming access to an (offline) ERM oracle, we show that there exists a computationally efficient online predictor that achieves a regret upper bounded by \(\tilde{O}(T^{\frac{3}{4}})\) for a finite-VC class, and upper bounded by \(\tilde{O}(T^{\frac{p+1}{p+2}})\) for a class with \(\alpha\) fat-shattering dimension \(\alpha^{-p}\). This provides the first known \emph{oracle-efficient} sublinear regret bounds for hybrid online learning with an \emph{unknown} feature generation process. In particular, it confirms a conjecture of~\cite{lazaric2012learning}. We then extend our result to the scenario of shifting distributions with \(K\) changes, yielding a regret of order \(\tilde{O}(T^{\frac{4}{5}}K^{\frac{1}{5}})\). Finally, we establish a regret of \(\tilde{O}((K^{\frac{2}{3}}(\log|\mathcal{H}|)^{\frac{1}{3}}+K)\cdot T^{\frac{4}{5}})\) for the contextual \(K\)-armed bandits with a finite policy set \(\mathcal{H}\), i.i.d. generated contexts from an \emph{unknown} distribution, and adversarially generated costs.
\end{abstract}

\sloppy
\section{Introduction}
We study the problem of \emph{hybrid} stochastic-adversary online learning, where the features are assumed to be sampled from an \emph{unknown} stochastic source while the labels are selected adversarially. Recent advancements~\cite{lazaric2009hybrid, rakhlin2012relax, haghtalab2020smoothed,haghtalab2022oracleefficient,haghtalab2022smoothed, block2022smoothed, wu2023online, wu2022expected} have demonstrated that such hybrid settings provide a fundamental paradigm shift beyond the classical \emph{worst-case} adversarial online setting to accommodate broader stochastic scenarios, while still preserves the capacity to handle various adversarial situations and maintain minimal assumptions on the expert class. 

We are interested in \emph{oracle-efficient} regret minimization methods as in ~\cite{kakade2005batch}. Here, we assume that the learner has access to an Empirical Risk Minimization (ERM) optimization oracle which, given any sequence of feature-label pairs, identifies the expert within the class that achieves the minimal cumulative loss. This effectively reduces the online learning problem to a batch learning problem, for which algorithms such as gradient descent have been highly successful in computing ERM optimization, even in complex classes like neural networks. Previous studies have applied this methodology in various online learning scenarios, including transductive online learning as in ~\cite{kakade2005batch}, online learning with \emph{smooth} adversary samples as in ~\cite{rakhlin2012relax, haghtalab2022oracleefficient, block2022smoothed}, and contextual bandits with a \emph{known} i.i.d. feature generation distribution as in ~\cite{pmlr-v48-rakhlin16, syrgkanis2016improved, banihashem2023an}. However, all of these works have assumed some form of access to a sampling oracle for the feature generation process. This may not be realistic when feature generation is costly, such as in medical data, or when the underlying probability law is unknown a priori. Other studies, like ~\cite{lazaric2009hybrid, wu2023online,wu2022expected}, do address scenarios with \emph{unknown} distributions, but provide only computationally inefficient prediction rules.

This paper initiates the study of \emph{oracle-efficient} hybrid online learning without assuming any access to the underlying probability law of the feature generation process. For the clarity of presentation, we will mainly focus on scenarios of online learning where features are generated by an \emph{unknown} i.i.d. process. However, we will also consider extensions to other scenarios, such as shifting distributions and contextual multi-armed bandits. Our approach also provides a general methodology that concentrates on the \emph{feature efficiency} in online learning.

\paragraph{Problem formulation.} Let $\mathcal{X}$ be an instance (feature) space and $\mathcal{H} \subset [0,1]^{\mathcal{X}}$ be a function class mapping $\mathcal{X} \rightarrow [0,1]$. We consider the following \emph{hybrid} online learning scenario. Nature selects an (unknown) distribution $\mu$ over $\mathcal{X}$ at the start of the game. At each time step $t$, Nature independently samples $\x_t \sim \mu$ and selects \emph{adversarially} a $y_t \in [0,1]$, but reveals only $\x_t$. A predictor then (randomly) generates $\hat{y}_t \in [0,1]$ based on $\x^t,y^{t-1}$, where $\x^t=\{\x_1,\cdots,\x_t\}$ and $y^{t-1}=\{y_1,\cdots,y_{t-1}\}$. Nature then reveals $y_t$, and the predictor incurs a loss $\ell(\hat{y}_t,y_t)$, for a predefined loss function $\ell: [0,1]^2 \rightarrow \mathbb{R}^+$. Here, we assume that the loss $\ell$ is \emph{convex} in its first argument and $L$-Lipschitz in \emph{both} arguments, e.g., the \emph{absolute loss} $\ell(\hat{y},y) = |\hat{y} - y|$. A prediction rule is a function $\Phi$ that takes inputs from $(\mathcal{X}\times [0,1])^* \times \mathcal{X}$ and generates a \emph{distribution} over $[0,1]$. For any prediction rule $\Phi$ and function class $\mathcal{H}$, we define the hybrid minimax regret as:
\begin{equation}
    \label{eq:hbreg}
    \scalebox{0.97}{$\begin{aligned}
    \Tilde{r}_T(\mathcal{H},\Phi) = \sup_{\mu} \mathbb{E}_{\x_1} \sup_{y_1 \in [0,1]} \mathbb{E}_{\hat{y}_1} \cdots \mathbb{E}_{\x_T} \sup_{y_T \in [0,1]} \mathbb{E}_{\hat{y}_T} \left[\sum_{t=1}^T \ell(\hat{y}_t,y_t) - \inf_{h \in \mathcal{H}} \sum_{t=1}^T \ell(h(\x_t),y_t)\right],\end{aligned}$}
\end{equation}
where $\x_t \sim \mu$ and $\hat{y}_t \sim \Phi(\x^t,y^{t-1})$ for $t\in [T]$. Our goal is to find an \emph{oracle-efficient} prediction rule $\Phi$ that minimizes $\Tilde{r}_T(\mathcal{H},\Phi)$.

\subsection{Results and Techniques}
In this work, we provide the first known oracle-efficient sub-linear regret bounds for the hybrid online learning with \emph{unknown} feature generation distributions. 
\begin{theorem}[Informal]
\label{thm:infmain1}
    Let \(\mathcal{H} \subset [0,1]^{\mathcal{X}}\) be a class with Rademacher complexity \(O(T^q)\). Then, there exists an oracle-efficient predictor with at most $O(\sqrt{T} \log T)$ calls to the ERM oracle that achieves the hybrid minimax regret of order \(\tilde{O}(T^{\frac{2-q}{3-2q}})\). In particular, for a VC-class, this implies a regret of \(\tilde{O}(T^{\frac{3}{4}})\), and for classes with \(\alpha\)-fat shattering dimension \(\alpha^{-p}\), it results in a regret of \(\tilde{O}(T^{\max\{\frac{3}{4},\frac{p+1}{p+2}\}})\), where $\tilde{O}$ hides poly-log factors.
\end{theorem}
To the best of our knowledge, the regret bounds presented in Theorem~\ref{thm:infmain1} are the first known \emph{oracle-efficient} sub-linear regrets for hybrid online learning with \emph{unknown} feature generation distributions and generic (non-parametric) hypothesis classes. In particular, our $\tilde{O}(T^{\frac{3}{4}})$ bound confirms a conjecture of~\cite{lazaric2012learning} regarding the oracle-efficient regret bounds for finite-VC classes under absolute loss~\footnote{They also obtained an $O(\sqrt{T\log T})$ bound using a computationally \emph{inefficient} covering-based approach.}. Note that, it was demonstrated by~\cite[Theorem 7]{block2022smoothed} that for \emph{known} feature generation distributions, an $\tilde{O}(T^{\max\{\frac{1}{2},\frac{p-1}{p}\}})$ regret bound is achievable for a class with an $\alpha$-fat shattering dimension of order $\alpha^{-p}$. However, to the best of our knowledge, such a chaining-based bound was not known for the \emph{unknown} distribution case, even in the information-theoretical sense. The closest comparison is the regret bound obtained in~\cite{wu2022expected} that \emph{matches} our $\tilde{O}(T^{\frac{p+1}{p+2}})$ bound, but established via an (inefficient) one-step covering approach.

 At a high level, our oracle-efficient prediction rule is based on the \emph{relaxation} and \emph{random play-out} techniques, as introduced in~\cite{rakhlin2012relax}. However, a distinguishing feature of our setup is that we are \emph{not} able to access the sampling oracle of the underlying feature generating process. Our main idea is to employ an \emph{epoch-based} approach. We partition the time horizon into a set of carefully \emph{designed} epochs. At each epoch, we \emph{estimate} the underlying distribution $\mu$ by $\hat{\mu}$ using samples observed in prior epochs. We then use the estimated distribution $\hat{\mu}$ to generate the \emph{hallucinated samples} as needed in the relaxation framework for the current epoch. 
 Observe, however, that the estimation $\hat{\mu}$ can arbitrarily deviate from $\mu$ under total variation, as we do not make any structural assumption on $\mu$, and the adversary \emph{knows} $\hat{\mu}$ when generating adversary samples. Therefore, the \emph{randomness matching} argument, as in~\cite{rakhlin2012relax,block2022smoothed}, will not work. To overcome this issue, we introduce a \emph{surrogate} relaxation based on $\hat{\mu}$ and relate it to the regret via a novel concept of \emph{approx-admissibility}, which is further controlled by a novel \emph{symmetrization} argument. The regret will then follow by carefully designing the epochs to balance the error introduced by the \emph{approx-admissibility} and the Rademacher complexity of the class restricted to each epoch.

 \paragraph{Tighter bounds for special cases.} Going beyond the general result in Theorem~\ref{thm:infmain1}, we show in Theorem~\ref{thm:obregret} (Appendix~\ref{app:obli}) that an oracle-efficient $\tilde{O}(T^{\max\{\frac{1}{2},\frac{p-1}{p}\}})$ regret is achievable for finite fat-shattering classes of order $p$ if we assume a weaker \emph{oblivious} adversary. This is tight upto poly-logarithmic factors. Furthermore, for the class $\mathcal{H}$ of all Lipschitz functions $[0,1]^d\rightarrow [0,1]$ we
establish in Theorem~\ref{th-lipschitz} the (optimal) regret $\tilde{O}(T^{\max\{\frac{1}{2},\frac{d-1}{d}\}})$ against the \emph{adaptive} adversary.

\paragraph{Shifting distributions.} Our next result drops the i.i.d. assumption and allows distributions to change over time. We assume that the number of changes is upper bounded by \(K\) and that the possible distributions and change points are completely unconstrained and unknown a priori. We show that the \emph{oracle-efficient} regret is upper bounded by \(\tilde{O}(T^{\frac{4}{5}}K^{\frac{1}{5}})\) for a finite VC class. Note that an \(O(\sqrt{KT\log T})\) bound was demonstrated by~\cite{wu2023online}. However, their algorithm relies on constructing an exponentially large covering, thus being computationally inefficient.

\paragraph{Contextual $K$-arm Bandits.} Finally, we establish an \(O((K^{\frac{2}{3}}(\log|\mathcal{H}|)^{\frac{1}{3}} + K\sqrt{\log K}) \cdot T^{\frac{4}{5}})\) \emph{oracle-efficient} regret bound for the contextual \(K\)-armed bandits with a finite policy set \(\mathcal{H}\), where the contexts are generated by an \emph{unknown} i.i.d. process and the costs are selected adversarially. The closest comparison to this result is the \(O(T^{\frac{2}{3}}(K\log|\mathcal{H}|)^{\frac{1}{3}})\) bound established in~\cite{banihashem2023an}, only for the \emph{known} i.i.d. distribution case. Notably, our result answers positively a question of~\cite{banihashem2023an} regarding relaxing the sampling access to the context distribution.

\section{Notation and Preliminaries}
\label{sec:def}

\paragraph{Oracle-Efficient Predictors.} We adopt the following \emph{mixed}-ERM oracle from~\cite{block2022smoothed}. Let $(\x_1, y_1), \cdots, (\x_m, y_m) \in \mathcal{X} \times [0,1]$, $C \in \mathbb{R}^{+}$, $\epsilon^n \in \{-1, +1\}^n$, and $\tilde{\x}^n \in \mathcal{X}^n$. The mixed ERM oracle is the task of finding 
$\inf_{h \in \mathcal{H}}\left\{\sum_{i=1}^m \ell(h(\x_i), y_i) + C\sum_{j=1}^n \epsilon_j h(\tilde{\x}_j)\right\}.$
Note that the unintuitive parts $\epsilon h(\tilde{\x})$ can be interpreted as an \emph{absolute} loss since $\epsilon h(\tilde{\x}) = |h(\tilde{\x}) - \frac{(1 - \epsilon)}{2}| - \frac{1 - \epsilon}{2}$ for $\epsilon \in \{+1, -1\}$. Therefore, the mixed ERM oracle is reduced to a regular (weighted) ERM oracle if $\ell$ is the absolute loss. Moreover, the weight $C$ can be understood as repeating the same sample $C$ times (rounding to an integer if necessary). We say a predictor $\Phi$ is \emph{oracle-efficient} if the running time of computing $\hat{y}_t \sim \Phi(\x^t, y^{t-1})$ is polynomial with respect to $t$ by accessing a mixed ERM oracle (with each oracle call treated as unit time) for any $\x^t, y^{t-1}$.

\paragraph{Adaptive v.s. Oblivious.} Note that our formulation in~(\ref{eq:hbreg}) assumes that the generation of $y_t$s is \emph{adaptive}, since the selection of $y_t$ at each time step depends on all prior information $\x^t$, $y^{t-1}$, and $\hat{y}^{t-1}$. For comparison, we also introduce a weaker notion of adversary, namely, the \emph{oblivious} adversary, which selects the $y_t$s based only on the current instance $\x_t$ (see Theorem~\ref{thm:obregret} in Appendix~\ref{app:obli}). It turns out that the adaptive nature of the adversary constitutes the main obstacle in our analysis.

\paragraph{Hybrid Contextual Bandits.} We now formulate the contextual $K$-arm bandits within our framework. Let $\mathcal{D}_K$ be the set of all probability distributions over $[K]$. A policy set $\mathcal{H}$ is a class of functions $\mathcal{X} \rightarrow [K]$. We consider the following bandit setup: Nature selects some $\mu$ at the start of the game. At each time step $t$, Nature samples $\x_t \sim \mu$ and selects \emph{adversarially} a cost vector $c_t \in [0,1]^K$, but reveals only $\x_t$. A predictor then selects a distribution $q_t \in \mathcal{D}_K$ based on the history observed thus far and samples $\hat{y}_t \sim q_t$. Nature reveals only $c_t[\hat{y}_t]$ which is also the predictor incurred loss. The goal is to find an \emph{oracle-efficient} prediction rule $\Phi: (\mathcal{X} \times [0,1])^* \times \mathcal{X} \rightarrow \mathcal{D}_K$ that minimizes:
%\begin{equation}
%        \label{eq:bandit}
        \scalebox{0.97}{$
\begin{aligned}
    \Tilde{r}_T^{\mathsf{bandit}}(\mathcal{H}, \Phi) = \sup_{\mu} \mathbb{E}_{\x_1} \sup_{c_1}\mathbb{E}_{\hat{y}_1} \cdots \mathbb{E}_{\x_T} \sup_{c_T}\mathbb{E}_{\hat{y}_T} \left[\sum_{t=1}^T \langle q_t,c_t\rangle - \inf_{h \in \mathcal{H}} \sum_{t=1}^T c_t[h(\x_t)]\right],
\end{aligned}$}
%\end{equation}
where $q_t=\Phi(\x^t, c_1[\hat{y}_1], \cdots, c_{t-1}[\hat{y}_{t-1}])$. Here, the ERM oracle is to find $\inf_{h \in \mathcal{H}} \sum_{i=1}^m \hat{c}_i[h(\x_i)]$ for any $(\x_1, \hat{c}_1), \cdots, (\x_m, \hat{c}_m) \in \mathcal{X} \times \mathbb{R}^K$, as~\cite{pmlr-v48-rakhlin16}.

\section{Oracle-Efficient Regret Bounds for Online Learning}
\label{sec:onlinebound}
In this section, we focus on bounding the hybrid minimax regret for online learning as in (\ref{eq:hbreg}) with a generic hypothesis class $\mathcal{H}$, using \emph{oracle-efficient} predictors by accessing to an mixed ERM oracle. We first recall the following standard notion of Rademacher complexity:
\begin{definition}
\label{def:rad}
    Let $\mathcal{H} \subset [0,1]^{\mathcal{X}}$ be a function class and $T \in \mathbb{N}^{+}$. The Rademacher complexity of $\mathcal{H}$ at horizon $T$ is defined to be
    \(\mathsf{Rad}_T(\mathcal{H}) = \sup_{\x^T \in \mathcal{X}^T} \mathbb{E}_{\epsilon^T} \left[\sup_{h \in \mathcal{H}} \sum_{t=1}^T \epsilon_t h(\x_t)\right],\)
    where $\epsilon_t$ is i.i.d. sampled from the uniform distribution over $\{\pm 1\}$.
\end{definition}

We are now ready to state our first main result:

\begin{theorem}
\label{thm:main1}
    Let $\mathcal{H} \subset [0,1]^{\mathcal{X}}$ be a class with $\mathsf{Rad}_T(\mathcal{H}) \le O(T^{q})$ for some $q \in [\frac{1}{2},1]$, and let $\ell$ be a $L$-Lipschitz loss that is convex in its first argument. Then there exists an \emph{oracle-efficient} prediction rule $\Phi$ with at most $O(L\sqrt{T}\log T)$ calls to the ERM oracle per round, such that
    \[\tilde{r}_T(\mathcal{H},\Phi) \le O\left(L\sqrt{\log (LT)} \cdot T^{\frac{2-q}{3-2q}}\right).\]
    In particular, for a binary-valued class with finite VC-dimension, we have
    \(\tilde{r}_T(\mathcal{H},\Phi) \le O(L\sqrt{\mathsf{VC}(\mathcal{H})\log (LT)} \cdot T^{\frac{3}{4}}),\)
    and for a real-valued class $\mathcal{H}$ with an $\alpha$-fat shattering dimension of order $\alpha^{-p}$ for $p>0$~\citep{alon1997scale}, we have
    \(\tilde{r}_T(\mathcal{H},\Phi) \le \tilde{O}(L T^{\max\{\frac{3}{4},\frac{p+1}{p+2}\}}).\)
\end{theorem}

\subsection{Regret Analysis with Side-Information}
\label{sec:side}
To establish Theorem~\ref{thm:main1}, we first consider a \emph{hypothetical} scenario where we assume the predictor has access to some \emph{side-information} $\x_{-N+1}^0$ sampled $i.i.d.$ from the same distribution $\mu$. It is crucial to note that this information is \emph{known} to the adversarial as well, i.e., the adversary's strategy could also depend on $\x_{-N+1}^0$, which turns out to be the main obstacle in our analysis~\footnote{We provide an analysis for the \emph{oblivious} adversaries in Appendix~\ref{app:obli}, yielding a tighter $\tilde{O}(T^{\max\{\frac{1}{2},\frac{p-1}{p}\}})$ regret when compared to Theorem~\ref{thm:main1}.}.

Formally, we consider the following learning game proceeds over a horizon of length $M$: At the start of the game, Nature selects an unknown distribution $\mu$ over $\mathcal{X}$, samples an $i.i.d.$ sample $\x_{-N+1}^0$ of size $N$ from $\mu$ and reveals $\x_{-N+1}^0$ to a predictor; At each time step $j\in [M]$, Nature samples $\x_j\sim \mu$ and selects \emph{adversarially} $y_j\in [0,1]$ (depends on $\x_{-N+1}^j$ and $\hat{y}^{t-1}$) but reveals only $\x_j$; The predictor then (randomly) generates $\hat{y}_j\in [0,1]$ based on $\x_{-N+1}^j$ and $y^{j-1}$; Nature reveals $y_j$ and the predictor incurs loss $\ell(\hat{y}_j,y_j)$, for some predefined convex and $L$-Lipschitz loss.

\paragraph{Predictor via \emph{surrogate} relaxation.} Let $\hat{\mu}_N$ be the \emph{empirical} distribution $\hat{\mu}_N=\frac{1}{N}\sum_{i=1}^N\delta_{\x_{-N+i}}$, where $\delta_{\x}$ is the Dirac measure on $\x$. For any time step $j\in [M]$ and horizon $M$ satisfying $M\le N/2$, we construct the following \emph{randomized} prediction rule:
\begin{itemize}
    \item[1.] Sample (internally) the \emph{hallucinated samples} $\Tilde{\x}_{j+1},\cdots,\Tilde{\x}_M$ from $\hat{\mu}_N$ \emph{without replacement}~\footnote{For technical reasons, we assume here that $\tilde{\x}_{j+1}^M$ is sampled from $\hat{\mu}_N$ \emph{without replacement}. Equivalently, $\tilde{\x}_{j+1}^M$ is sampled uniformly from all (permuted) \emph{subseqeunces} of $\x_{-N+1}^0$ of length $M-j$.} and $\epsilon_{j+1},\cdots, \epsilon_M$ $i.i.d.$ from the uniform distribution over $\{-1,+1\}$;
    \item[2.] Make prediction
    \begin{equation}
    \scalebox{0.97}{% 
    $\begin{aligned}
    \hat{y}_j &= \mathop{\mathrm{arg\,min}}\limits_{\hat{y}\in [0,1]}\sup_{y\in [0,1]}\left\{\ell(\hat{y},y)+\sup_{h\in \mathcal{H}}\left[2L\sum_{i=j+1}^M\epsilon_i h(\Tilde{\x}_i)-\ell(h(\x_j),y)-\sum_{i=1}^{j-1}\ell(h(\x_i),y_i)\right]\right\}.
    \end{aligned}$}
\label{eq:predictor}
\end{equation}

\end{itemize}

Note that the main difference from the classical \emph{random play-out} techniques such as in~\cite{rakhlin2012relax,block2022smoothed} is that the hallucinated samples are generated from $\hat{\mu}_N$ instead of $\mu$. Crucially, our sampling is performed \emph{without replacement} (not $i.i.d.$), which is essential for our following analysis (Lemma~\ref{lem:dis2rad}). More generally, one may also replace the estimation $\hat{\mu}_N$ with other estimation rules instead of the empirical distribution we used here. This could provide tighter bounds if the distribution $\mu$ is well structured, see Section~\ref{sec:special}. The following lemma shows that the predictor $\hat{y}_j$ can be computed \emph{efficiently} by accessing to an mixed-ERM oracle.
\begin{lemma}
\label{lem:comput}
    The predictor $\hat{y}_j$ can be computed upto error $\pm\frac{1}{L\sqrt{M}}$  by making at most $O(L\sqrt{M}\log M)$ mixed-ERM oracle calls. Moreover, for binary valued class $\mathcal{H}$ with $y\in \{0,1\}$ and \emph{absolute} loss, we need only $2$ (regular) ERM orcale calls to compute $\hat{y}_j$ exactly.
\end{lemma}
\begin{proof}
    Clearly, a naive approach for discretizing both $\hat{y}$ and $y$ with scale $\frac{1}{L\sqrt{M}}$ yields an algorithm with $L^2 M$ oracle calls. The $O(L\sqrt{M}\log M)$ bound follows from~\cite[Thm 7]{block2022smoothed} leveraging the convexity on $\hat{y}$. The second part follows from the relation $\epsilon h(\tilde{\x}) = |h(\tilde{\x}) - \frac{(1 - \epsilon)}{2}| - \frac{1 - \epsilon}{2}$ for $\epsilon \in \{+1, -1\}$ and the second assertion of ~\cite[Thm 7]{block2022smoothed}.
\end{proof}

\paragraph{Analysis of the regret.} Denote by $\Phi$  the prediction rule derived from (\ref{eq:predictor}). We consider the following analogous hybrid minimax regret as in (\ref{eq:hbreg}) with the additional \emph{side-information}:
\begin{equation}
    \label{eq:side}\scalebox{0.97}{$\begin{aligned}
    \tilde{r}_{M,N}^{\mathsf{side}}(\mathcal{H},\Phi)=\sup_{\mu}\mathbb{E}_{\x_{-N+1}^0}\mathbb{E}_{\x_1}\sup_{y_1}\mathbb{E}_{\hat{y}_1}\cdots \mathbb{E}_{\x_M}\sup_{y_M}\mathbb{E}_{\hat{y}_M}\left[\sum_{j=1}^M\ell(\hat{y}_j,y_j)-\inf_{h\in \mathcal{H}}\sum_{j=1}^M\ell(h(\x_j),y_j)\right],\end{aligned}$}
\end{equation}
where the randomness of $\hat{y}_j$s is over the $\tilde{\x}$'s and $\epsilon$'s as in (\ref{eq:predictor}), while $\x_j$s are sampled $i.i.d.$ from $\mu$.

To proceed, we first introduce the following key concept. Let $(\x_1,y_1),\cdots,(\x_M,y_M)\in\mathcal{X}\times [0,1]$ be any realization of the feature-label pairs. We write $L_j^h=\sum_{i=1}^j\ell(h(\x_i),y_i)$ to simplify our discussion. The \emph{surrogate} relaxation is defined as
\begin{equation}
\label{eq:relax}
    R_j=\mathbb{E}_{\Tilde{\x},\epsilon}\left[\sup_{h\in \mathcal{H}}2L\sum_{i=j+1}^M\epsilon_ih(\tilde{\x}_i)-L_j^h\right],
\end{equation}
where $\tilde{\x}_i$s and $\epsilon_i$s are generated the same way as in (\ref{eq:predictor}). We also define the following variation that replaces the single $\tilde{\x}_{j+1}$ with a sample $\x\sim \mu$:
\begin{equation}
\label{eq:relaxshit}
    \tilde{R}_j=\mathbb{E}_{\x\sim \mu}\mathbb{E}_{\Tilde{\x},\epsilon}\left[\sup_{h\in \mathcal{H}}2L\epsilon_{j+1}h(\x)+2L\sum_{i=j+2}^M\epsilon_ih(\tilde{\x}_i)-L_j^h\right].
\end{equation}

Note that the main technique for proving the relaxation based regret bounds, such as~\cite{rakhlin2012relax}, is through the concept of \emph{admissibility}, which essentially asserts that
$\mathbb{E}_{\x_j}\sup_{y_j}\mathbb{E}_{\epsilon,\tilde{\x}}\left[\ell(\hat{y}_j,y_j)+R_j\right]\le R_{j-1}.$
However, a major technical step for establishing such an result is based on the so-called \emph{randomness matching} argument by leveraging the fact that the \emph{hallucinated samples} used to define the relaxation are the same as the actual feature generating process. This, unfortunately, is not true in our case since the empirical distribution $\hat{\mu}_N$ can deviate arbitrarily from $\mu$ under total variation, regardless of how large the sample size $N$ is. We instead establish the following \emph{approx-admissibility} of our surrogate relaxation, with the proof deferred to Appendix~\ref{sec:appproflem5}.

\begin{lemma}[Approx-Admissibility]
\label{lem:admiss}
   Let $\hat{y}_j$ be as in (\ref{eq:predictor}), then for all $j\in [M]$ we have:
    \begin{equation}
\label{eq1lem1}        \mathbb{E}_{\x_j}\sup_{y_j}\mathbb{E}_{\epsilon,\tilde{\x}}\left[\ell(\hat{y}_j,y_j)+R_j\right]\le \tilde{R}_{j-1}.
    \end{equation}  
\end{lemma}

We are now ready to state our first main technical lemma of this section, which follows from Lemma~\ref{lem:admiss} by a "backward tracing" argument. The detailed proof is deferred to Appendix~\ref{sec:prooflem6}.
\begin{lemma}[Regret Bound via Approx-Admissibility]
\label{lem:main2}
    Let $\Phi$ be the predictor as in (\ref{eq:predictor}). Then for any class $\mathcal{H}\subset [0,1]^{\mathcal{X}}$ with a convex and $L$-Lipschitz loss $\ell$, we have
    \begin{equation}
        \label{eq:telescop}
        \tilde{r}_{M,N}^{\mathsf{side}}(\mathcal{H},\Phi)\le \mathbb{E}_{\x_{-N+1}^0}\left[\tilde{R}_0+\sum_{j=1}^{M-1}\mathbb{E}_{\x^j}\sup_{y^j}(\tilde{R}_j-R_j)\right],
    \end{equation}
    where $\x_{-N+1}^M$ are sampled $i.i.d.$ from $\mu$ and $R_j$, $\tilde{R}_j$ are defined as in (\ref{eq:relax}) and (\ref{eq:relaxshit}).
\end{lemma}
\begin{remark}
    Note that the decomposition presented in Lemma~\ref{lem:main2} holds whenever the approx-admissibility condition of Lemma~\ref{lem:admiss} is satisfied. We believe this could be applicable to a broader set of problems and is of independent interest.
\end{remark}

\paragraph{Bounding the relaxations.} As we have demonstrated in Lemma~\ref{lem:main2}, the regret $\tilde{r}_{M,N}^{\mathsf{side}}(\mathcal{H},\Phi)$ can be upper bounded by $\tilde{R}_0$ and the discrepancies between $R_j$ and $\tilde{R}_j$. Clearly, by the definition of $\tilde{R}_j$, we have $\tilde{R}_0\le 2L \mathsf{Rad}_M(\mathcal{H}),$ where $\mathsf{Rad}_M(\mathcal{H})$ is the  Rademacher complexity of $\mathcal{H}$ as in Definition~\ref{def:rad}. To bound the discrepancies, for any $j\in [M-1]$, $\x^j,\tilde{\x}_{j+2}^M\in \mathcal{X}^*$, $\epsilon_{j+1}^M\in \{\pm 1\}^*$ and $y^j\in [0,1]^j$, we define the following function:
\begin{equation}
    \label{eq:ffunc}
    f_{\x^j,\tilde{\x}_{j+2}^M,\epsilon_{j+1}^M,y^j}(\x)=\sup_{h\in \mathcal{H}}\left\{2L\epsilon_{j+1}h(\x)+2L\sum_{i=j+2}^M\epsilon_i h(\tilde{\x}_i)-L_j^h\right\}.
\end{equation}
The following fact is a consequence of our definitions. 
\begin{fact}
\label{fact1}
    We have $R_j=\mathbb{E}_{\tilde{\x},\epsilon}\left[f_{\x^j,\tilde{\x}_{j+2}^M,\epsilon_{j+1}^M,y^j}(\tilde{\x}_{j+1})\right]$ and $\tilde{R}_j=\mathbb{E}_{\x\sim \mu}\mathbb{E}_{\tilde{\x},\epsilon}\left[f_{\x^j,\tilde{\x}_{j+2}^M,\epsilon_{j+1}^M,y^j}(\x)\right]$.
\end{fact}

Let $\z_j=(\x^j,\tilde{\x}_{j+2}^M,\epsilon_{j+1}^M)$. We now observe the following key properties of the functions $f_{\z_j,y^j}(\x)$, which demonstrates that $f_{\z_j,y^j}(\x)$ has \emph{sensitivity} upper bounded by $4L$ and is Lipschitz on $y^j$. We refer to Appendix~\ref{sec:omit} for the detailed proof.
\begin{proposition}
\label{prop1}
    For any $\z_j$ and $y^j$, we have
        $\sup_{\x,\x'}|f_{\z_j,y^j}(\x)-f_{\z_j,y^j}(\x')|\le 4L.$
    Moreover, for all $\z_j$, $\x$ and $y^j,y'^j\in [0,1]^j$, we have $|f_{\z_j,y^j}(\x)-f_{\z_j,y'^j}(\x)|\le jL||y^j-y'^j||_{\infty}$.
\end{proposition}

Note that Proposition~\ref{prop1} and Fact~\ref{fact1} immediately imply that $\tilde{R}_j-R_j\le 4L||\mu-\hat{\mu}_N||_{\mathsf{TV}}$~\footnote{Technically, we need to assume here that $\tilde{\x}$s are sampled $i.i.d.$ from $\hat{\mu}_N$, i.e., \emph{with replacement}.}. Unfortunately, we are unable to bound the total variation distance $||\mu-\hat{\mu}_N||_{\mathsf{TV}}$ due to the lack of any structure we impose on $\mu$. We instead establish the following key technical result, which bounds the discrepancies via a Rademacher sum of the functions $f_{\z_j,y^j}$. This result constitutes the main technical ingredient in our following analysis.
\begin{lemma}
\label{lem:dis2rad}
    For all $j\in [M-1]$, $M\le N/2$ and $B=N-M+j+1$, we find
    \begin{equation}
        \mathbb{E}_{\x_{-N+1}^0}\mathbb{E}_{\x^j}\sup_{y^j}(\tilde{R}_j-R_j)\le \sup_{\x_{-N+1}^{-N+B},\x'^{B},\z_j}\mathbb{E}_{\epsilon'^B}\left[\sup_{y^j}\frac{1}{B}\sum_{i=1}^{B}\epsilon_i'(f_{\z_j,y^j}(\x_i')-f_{\z_j,y^j}(\x_{-N+i}))\right],
    \end{equation}
    where $\x_{-N+1}^{-N+B},\x'^{B},\z_j$ run over all possible values and $\epsilon'^B$ is distributed uniformly over $\{\pm 1\}^B$.
\end{lemma}
\begin{proof}[Sketch]
We highlight only the main idea here
    and refer to Appendix~\ref{sec:applem8} for the complete details. By Fact~\ref{fact1}, we can upper bound the discrepancies by $\mathbb{E}_{\x_{-N+1}^0}\mathbb{E}_{\z_j}\sup_{y^j}[\mathbb{E}_{\x\sim \mu}[f_{\z_j,y^j}(\x)]-\mathbb{E}_{\tilde{\x}_{j+1}}[f_{\z_j,y^j}(\tilde{\x}_{j+1})]]$, where $\z_j=(\x^j,\tilde{\x}_{j+2}^M,\epsilon^M_{j+1})$. Note that $\tilde{\x}_{j+1}^M$ is sampled uniformly from $\x_{-N+1}^0$ \emph{without replacement} as in (\ref{eq:predictor}). Therefore, the randomness of $\tilde{\x}_{j+1}^M$ can be described as follows: we first sample $\tilde{\x}_{j+2}^M$ from $\x_{-N+1}^0$ and then sample $\tilde{\x}_{j+1}$ \emph{uniformly} from the remaining samples in $\x_{-N+1}^0$. Now, the key observation is that, by symmetries of $\x_{-N+1}^0$ (which are $i.i.d.$), we can \emph{fix} $\tilde{\x}_{j+2}^M$ being the last $M-j-1$ samples in $\x_{-N+1}^0$. Therefore, we have $\mathbb{E}_{\tilde{\x}_{j+1}}[f_{\z_j,y^j}(\tilde{\x}_{j+1})]=\frac{1}{B}\sum_{i=1}^Bf_{\z_j,y^j}(\x_{-N+i})$, where $B=N-M+j+1$. Since $\z_j$ is \emph{decoupled} from $\x_{-N+1}^{-N+B}$ by our construction, we obtain the upper bound $\mathbb{E}_{\z_j}\mathbb{E}_{\x_{-N+1}^{-N+B}}\sup_{y^j}[\mathbb{E}_{\x\sim \mu}[f_{\z_j,y^j}(\x)]-\frac{1}{B}\sum_{i=1}^Bf_{\z_j,y^j}(\x_{-N+i})]$. The lemma then follows by \emph{symmetrization} with $\mathbb{E}_{\x\sim \mu}[f_{\z_j,y^j}(\x)]$ (see Appendix~\ref{sec:applem8}).
\end{proof}

For any $j\in [M-1]$ and $\z_j$ as above, we define the following function class~\footnote{Note that the "complexity" of $\mathcal{G}_{\mathbf{z}_j}$ arises from the $y^j\in [0,1]^j$.}:
\begin{equation}
    \label{eq:fclass}
    \mathcal{G}_{\z_j}=\{g_{\z_j,y^j}(\x,\x')\overset{\text{def}}{=}f_{\z_j,y^j}(\x')-f_{\z_j,y^j}(\x):y^j\in [0,1]^j,(\x,\x')\in \mathcal{X}^2\}.
\end{equation}
Lemma~\ref{lem:dis2rad} essentially states that the discrepancy between $R_j$ and $\tilde{R}_j$ is upper bounded by the Rademacher complexity of the class $\mathcal{G}_{\z_j}$ as
%\begin{equation}
%\label{eq:dis2rad2}
    $\mathbb{E}_{\x_{-N+1}^0}\mathbb{E}_{\x^j}\sup_{y^j}(\tilde{R}_j-R_j)\le \sup_{\z_j}\frac{1}{B}\mathsf{Rad}_{B}(\mathcal{G}_{\z_j}).$
%\end{equation}

The following lemma provides a useful bound on such Rademacher complexities.
\begin{lemma}
\label{lem:rad2bound}
    Let $\mathcal{G}_{\z_j}$ be as in (\ref{eq:fclass}), $M\le N/2$ and $B=N-M+j+1$. Then
    \begin{equation}
        \label{eq:radbound}
        \sup_{\z_j}\frac{1}{B}\mathsf{Rad}_{B}(\mathcal{G}_{\z_j})\le O\left(\sqrt{\frac{jL^2\log(jLB)}{B}}\right)\le O\left(\sqrt{\frac{2jL^2\log(jLN/2)}{N}}\right).
    \end{equation}
\end{lemma}
\begin{proof}
    Let $\mathcal{C}\subset [0,1]^j$ be a covering of $[0,1]^j$ with $L_{\infty}$ radius $\frac{1}{jLB}$. We have $|\mathcal{C}|\le (jLB)^j$. By the second part of Proposition~\ref{prop1}, we know that the class $\mathcal{G}'_{\z_j}=\{g_{\z_j,y^j}:y^j\in \mathcal{C}\}$ forms a uniform $L_{\infty}$-covering of $\mathcal{G}_{\z_j}$ with radius $\frac{2}{B}$. Therefore, $\frac{1}{B}\mathsf{Rad}_B(\mathcal{G}_{\z_j})\le \frac{1}{B}\mathsf{Rad}_{B}(\mathcal{G}_{\z_j}')+\frac{2}{B}$. The first inequality then follows by a simple application of Massart’s lemma~\cite[Lemma 26.8]{shalev2014understanding} over $\mathcal{G}'_{\z_j}$, since $|\mathcal{G}'_{\z_j}|\le |\mathcal{C}|\le (jLB)^j$ and $\sup_{(\x,\x')\in \mathcal{X}^2}\{g_{\z_j,y^j}(\x,\x')\}\le 4L$ for all $g_{\z_j,y^j}\in \mathcal{G}_{\z_j}$ due to the first part of Proposition~\ref{prop1}. The second inequality is implied by that $B\ge N/2$ and the fact that the function $\frac{\log B}{B}$ is monotone decreasing.
\end{proof}

Putting everything together, we arrive at:
\begin{theorem}
\label{thm:main2}
    Let $\Phi$ be the predictor as in (\ref{eq:predictor}) and $M\le N/2$. Then for any class $\mathcal{H}\subset [0,1]^{\mathcal{X}}$ with a convex and $L$-Lipschitz loss $\ell$, the predictor $\Phi$ can be computed efficiently with access to at most $O(L\sqrt{M}\log M)$ mixed-ERM oracle calls per round such that
    \begin{equation}
        \label{eq:reg2rad}
        \tilde{r}_{M,N}^{\mathsf{side}}(\mathcal{H},\Phi)\le 2L\mathsf{Rad}_M(\mathcal{H})+\sqrt{M}+O\left(\sqrt{\frac{M^3L^2\log(MLN)}{N}}\right).
    \end{equation}
\end{theorem}
\begin{proof}
    The regret bound follows directly from Lemma~\ref{lem:comput}, Lemma~\ref{lem:main2} and Lemma~\ref{lem:dis2rad}. We then invoke Lemma~\ref{lem:rad2bound} to bound the discrepancies by noticing that $j\le M$.
\end{proof}
\begin{remark}
    Note that Theorem~\ref{thm:main2} shows that if $N\gg M^2\log M$ then the regret with side-information is reduced to the Rademacher complexities of $\mathcal{H}$, and thus matches the case when the distribution is known in advance. However, in reality such side-information is not available for the unknown distribution case, which can only be obtained from prior samples.
\end{remark}

\subsection{Proof of Theorem~\ref{thm:main1}: the Epoch Approach}
\label{sec:epoch}
We are now  equipped with all the technical tools to prove Theorem~\ref{thm:main1}, with the only missing ingredient of constructing the \emph{side-information}. For this purpose, we employ an \emph{epoch-based} approach, resembling those used in~\cite{lazaric2009hybrid,wu2023online}, but in a completely different context. We partition the time horizon into epochs, with  epoch $n$ of length $M(n)$. Let
$S(n)=\sum_{i=1}^{n-1} M(i)$
be the total time steps after $n-1$ epochs. We will use the features observed upto time $S(n)$ as the side-information introduces in Section~\ref{sec:side} and apply the predictor constructed in (\ref{eq:predictor}) to make the prediction during the $n$th epoch.
\newpage
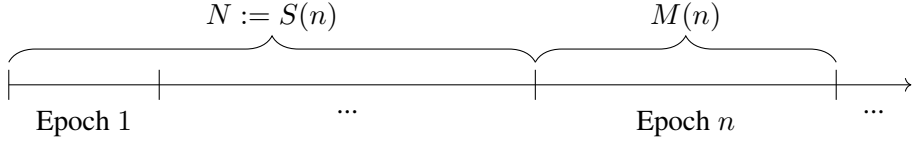
\begin{figure}[h]
    \centering
\begin{tikzpicture}

% Draw horizontal line
\draw[->] (0,0) -- (12,0);

% Draw vertical lines for epochs
\foreach \x in {0,2,7,11} {
  \draw (\x,-0.2) -- (\x,0.2);
}

% Label epochs
\node[align=center, below] at (1, -0.2) {Epoch $1$};
\node[align=center, below] at (4.5, -0.2) {...};
\node[align=center, below] at (9, -0.2) {Epoch $n$};
\node[align=center, below] at (11.5, -0.2) {...};

% Draw smaller vertical lines for time steps
%\foreach \x in {0.5, 1, 1.5, 2, 2.5, 3.5, 4, 4.5, 5, 5.5, 6.5, 7, 7.5, 8, 8.5, 9.5, 10, 10.5, 11, 11.5} {
 % \draw (\x,-0.1) -- (\x,0.1);
%}

% Label time steps in Epoch 1
%\node[align=center, above] at (0, 0.1) {$t_0$};
%\node[align=center, above] at (2, 0.1) {$t_1$};
%\node[align=center, above] at (5, 0.1) {$t_{n-1}$};
%\node[align=center, above] at (11, 0.1) {$t_n$};
%\node[align=center, above] at (0,0.6) ;
%\node[align=center, above] at (2,0.6) {$\mathcal{H}^1$};
%\node[align=center, above] at (5,0.6) {$\mathcal{H}^{s-1}$};
%\node[align=center, above] at (8, 0.2) {$E_{t_s,t_{s+1}}\ge N$};

\draw [decorate,decoration={brace,amplitude=10pt}] 
(0,0.3) -- (7,0.3) node [black,midway,yshift=0.6cm] {$N:=S(n)$};

\draw [decorate,decoration={brace,amplitude=10pt}] 
(7,0.3) -- (11,0.3) node [black,midway,yshift=0.6cm] {$M(n)$};

% ... Add more time steps for other epochs if needed

\end{tikzpicture}
%    \caption{Illustration of epoch partition.}
%    \label{fig:enter-label}
\end{figure}
%\fi
To this end, our main technical part is to \emph{optimize} the epoch length $M(n)$ that balances the trade-off in (\ref{eq:reg2rad}) and achieving the minimal total regret. Let $\Phi$ be the predictor derived from (\ref{eq:predictor}), which we write  as $\Phi(\x_{-N+1}^0,\x^j,y^{j-1})$ for the side-information $\x_{-N+1}^0$, features $\x^j$ and labels $y^{j-1}$ observed thus far. We define the following \emph{epoch} predictor $\Psi$: for any epoch $n$ and time step $j$ during such epoch, we set
\begin{equation}
    \label{eq:eppredictor}
    \Psi(\x^{S(n)+j},y^{S(n)+j-1})=\Phi\left(\x^{S(n)},\x_{S(n)+1}^{S(n)+j},y_{S(n)+1}^{S(n)+j-1}\right).
\end{equation}
Let $S^{-1}(T)$ be the largest number $n$ such that  $S(n)<T$. The following lemma upper bounds the hybrid minimax regret (\ref{eq:hbreg}) of $\Psi$ using the regrets with side information (\ref{eq:side}) incurred by $\Phi$. Note that this is \emph{not} immediately obvious since we have \emph{reused} the side-information among different epochs.
\begin{lemma}
\label{lem:side2univer}
    For any $\mathcal{H}$ and convex $L$-Lipschitz loss $\ell$, we have
    $$\tilde{r}_T(\mathcal{H},\Psi)\le \sum_{n=1}^{S^{-1}(T)}\tilde{r}_{M(n),S(n)}^{\mathsf{side}}(\mathcal{H},\Phi).$$
%    where $S^{-1}(T)$ is the minimal number $n$ satisfies $S(n)\ge T$.
\end{lemma}
\begin{proof}
    Define the operator $\mathbb{Q}_i^j\equiv \mathbb{E}_{\x_i}\sup_{y_i}\mathbb{E}_{\hat{y}_i}\cdots \mathbb{E}_{\x_j}\sup_{y_j}\mathbb{E}_{\hat{y}_j}$, where $\hat{y}_t\sim \Psi(\x^t,y^{t-1})$ for all $t\in [T]$. We have (truncate the last $S(n+1)$ above $T$ if necessary):
   \begin{align*}
    \tilde{r}_T(\mathcal{H},\Psi)&=\mathbb{Q}_1^T\sup_{h\in \mathcal{H}}\left[\sum_{n=1}^{S^{-1}(T)}\sum_{j=S(n)+1}^{S(n+1)}\ell(\hat{y}_j,y_j)-\ell(h(\x_j),y_j)\right]\\
            &\overset{(a)}{\le} \sum_{n=1}^{S^{-1}(T)} \mathbb{Q}_1^T\sup_{h\in \mathcal{H}}\left[\sum_{j=S(n)+1}^{S(n+1)}\ell(\hat{y}_j,y_j)-\ell(h(\x_j),y_j)\right]\\
            &\overset{(b)}{=}\sum_{n=1}^{S^{-1}(T)}\mathbb{E}_{\x^{S(n)}}\mathbb{Q}_{S(n)+1}^{S(n+1)}\sup_{h\in \mathcal{H}}\left[\sum_{j=S(n)+1}^{S(n+1)}\ell(\hat{y}_j,y_j)-\ell(h(\x_j),y_j)\right]\\
            &\overset{(c)}{=}\sum_{n=1}^{S^{-1}(T)}\tilde{r}_{M(n),S(n)}^{\mathsf{side}}(\mathcal{H},\Phi),
    \end{align*}
    where $(a)$ follows by $\sup(A+B)\le \sup A+\sup B$ and linearity of expectation; $(b)$ follows since $\hat{y}_j$ depends only on $\x^j$ and $y_{S(n)}^j$ for $j\in (S(n),S(n+1)]$; $(c)$ follows by definition.
\end{proof}

%We are now ready to prove our main result Theorem~\ref{thm:infmain1}.

\begin{proof}[Proof of Theorem~\ref{thm:main1}]
    Assume $\mathsf{Rad}_T(\mathcal{H})\le O(T^q)$ for some $q\in [\frac{1}{2},1]$. By Theorem~\ref{thm:main2} and $M(n),S(n)\le T$ we have
    $$\tilde{r}_{M(n),S(n)}^{\mathsf{side}}(\mathcal{H},\Phi)\le O\left(LM(n)^q+\sqrt{\frac{M(n)^3L^2\log(LT^2)}{S(n)}}\right).$$
    Let $M(n)=n^{\alpha}$ for some $\alpha>0$ to be determined later. We have $S(n)=\sum_{i=1}^{n-1} i^{\alpha}=\Theta(n^{\alpha+1})$ by integration approximation, and $S^{-1}(T)\le O(T^{1/(\alpha+1)})$. This implies that
    $\tilde{r}_{M(n),S(n)}^{\mathsf{side}}(\mathcal{H},\Phi)\le O(Ln^{\alpha q}+L\sqrt{\log(LT^2)}n^{\alpha-\frac{1}{2}}).$
    By Lemma~\ref{lem:side2univer} and integration approximation again, we conclude
    \begin{equation}
    \label{eq:thm2proof}
        \tilde{r}_T(\mathcal{H},\Psi)\le O\left(LT^{\frac{\alpha q+1}{\alpha+1}}+L\sqrt{\log(LT^2)}T^{\frac{\alpha+\frac{1}{2}}{\alpha+1}}\right).
    \end{equation}
    Optimizing $\arg\min_{\alpha>0}\max\{\frac{\alpha q+1}{\alpha+1},\frac{\alpha+\frac{1}{2}}{\alpha+1}\}$, we find (\ref{eq:thm2proof}) is minimized when taking $\alpha=\frac{1}{2(1-q)}$. Plugging back to (\ref{eq:thm2proof}), we find
    $\tilde{r}_T(\mathcal{H},\Psi)\le O\left(L\sqrt{\log(LT)}T^{\frac{2-q}{3-2q}}\right).$
    This completes the proof of the first part. The second and third parts follow by the facts that $\mathsf{Rad}_T(\mathcal{H})\le O(\sqrt{\mathsf{VC}(\mathcal{H})T})$ for finite-VC class~\citep[Example 5.24]{wainwright2019high}, and $\mathsf{Rad}_T(\mathcal{H})\le \tilde{O}(T^{\max\{\frac{1}{2},\frac{p-1}{p}\}})$ for classes with $\alpha$-fat shattering dimension of order $\alpha^{-p}$~\citep{block2022smoothed}. This completes the proof and the big-O notations and $M(n)\le S(n)/2$ are justified by noting that $\alpha\ge 1$ since $q\ge\frac{1}{2}$.
\end{proof}

\subsection{Tighter Bounds for Special Classes}
\label{sec:special}

As demonstrated in Section~\ref{sec:side}, the main technical obstacle for analyzing the regret is to upper bound the discrepancies between $\tilde{R}_j$ and $R_j$ as in Lemma~\ref{lem:main2}. It was shown in Lemma~\ref{lem:dis2rad} that such discrepancies can be upper bounded by the Rademacher complexity of the class $\mathcal{G}_{\z_j}$ in (\ref{eq:fclass}). We demonstrate in this section how to leverage the \emph{structural} information of $\mathcal{G}_{\z_j}$ leading to tighter regret bounds for certain special classes, when compared to the general bounds from Theorem~\ref{thm:main1}.

\paragraph{Binary valued classes.} Let $\mathcal{H}\subset \{0,1\}^{\mathcal{X}}$ be a binary valued class and $\ell(\hat{y},y)=|\hat{y}-y|$. For any given $\z_j$ (assume, w.l.o.g., $\epsilon_{j+1}=1$) and $y^j\in \{0,1\}^j$, the function $f_{\z_j,y^j}$ can be expressed as
    $f_{\z_j,y^j}(\x)=\sup_{h} \{2h(\x)+F(h)\}$ (see definition in (\ref{eq:ffunc})),
    where $F(h)$ is a discrete valued function taking values in $[-2M,2M]$. Define
    $\mathcal{H}^0=\left\{h\in \mathcal{H}:F(h)=\sup_{h'\in \mathcal{H}}F(h')\right\}\text{ and }\mathcal{H}^1=\left\{h\in \mathcal{H}:F(h)=\sup_{h'\in \mathcal{H}}F(h')-1\right\}.$
    Let $h^0(\x)=\sup_{h\in \mathcal{H}^0}\{h(\x)\}$, $h^1(\x)=\sup_{h\in \mathcal{H}^1}\{h(\x)\}$ and $\hat{h}=\arg\max_{h\in \mathcal{H}}F(h)$. The following \emph{structural} characterization of $f_{\z_j,y^j}$ holds (see Appendix~\ref{sec:omit} for proof):
    \begin{fact}
    \label{fact2}
    For any $\x$, if $h^0(\x)=1$ then $f_{\z_j,y^j}(\x)=F(\hat{h})+2$; if $h^0(\x)=0$ and $h^1(\x)=1$ then $f_{\z_j,y^j}(\x)=F(\hat{h})+1$ ; else $f_{\z_j,y^j}(\x)=F(\hat{h})$.
%        \begin{equation}
%    \label{eq:fstructure}
%        f_{\z_j,y^j}(\x)=\begin{cases}
%        F(\hat{h})+2,~\text{if }h^0(\x)=1\\
%        F(\hat{h})+1,~\text{if }h^0(\x)=0\text{ and }h^1(\x)=1\\
%        F(\hat{h}), \text{ else}
%    \end{cases}.
%    \end{equation}
    \end{fact}
%    We are now ready to state the main result of this section.
\begin{theorem}
\label{cor1}
    Let $\mathcal{H}\subset \{0,1\}^{\mathcal{X}}$, $\mathcal{F}^{\mathsf{u}}=\{f_{\mathcal{H}'}(\x)=\sup_{h\in \mathcal{H}'}\{h(\x)\}:\mathcal{H}'\subset \mathcal{H}\}$, $\mathcal{F}^{\mathsf{i}}=\{f_{\mathcal{H}'}(\x)=\inf_{h\in \mathcal{H}'}\{h(\x)\}:\mathcal{H}'\subset \mathcal{H}\}$ be two classes of functions and $\ell$ be the absolute loss. Then there exists an oracle-efficient predictor $\Phi$ satisfying
    $\tilde{r}_T(\mathcal{H},\Phi)\le O(\sqrt{\max\{\mathsf{VC}(\mathcal{F}^{\mathsf{u}}), \mathsf{VC}(\mathcal{F}^{\mathsf{i}})\}T})$.
\end{theorem}
\begin{proof}
         Assume, w.o.l.g., $\epsilon_{j+1}=1$. The functions $h^0,h^1$ as in Fact~\ref{fact2} are within $\mathcal{F}^{\mathsf{u}}$. For \emph{any} $\x^{2N}\in \mathcal{X}^{2N}$ and $\hat{\mu}$ uniform over $\x^{2N}$, there exists a $\gamma$-cover $\mathcal{C}_{\gamma}$ of $\mathcal{F}^{\mathsf{u}}$ under distance $d_{\hat{\mu}}(f_1,f_2)\overset{\text{def}}{=}\mathrm{Pr}_{\x\sim \hat{\mu}}[f_1(\x)\not=f_2(\x)]$ such that $|\mathcal{C}_{\gamma}|\le O(\frac{1}{\gamma^{\mathsf{VC}(\mathcal{F}^{\mathsf{u}})}})$~\citep{haussler1995sphere}. By Fact~\ref{fact2}, there exists a function $\mathcal{T}:(\mathcal{F}^{\mathsf{u}})^2\rightarrow \{0,1,2\}^{\mathcal{X}}$ such that for \emph{any} $f_{\z_j,y^j}$, there exist $h^0,h^1\in \mathcal{F}^{\mathsf{u}}$ such that $f_{\z_j,y^j}(\x)=\mathcal{T}(h^0(\x),h^1(\x))+c_{\z_j,y^j}$, where $c_{\z_j,y^j}=F(\hat{h})$ as in Fact~\ref{fact2}. Therefore, the function class $\mathcal{C}'\overset{\text{def}}{=}\{\mathcal{T}(h^0,h^1):h^0,h^1\in \mathcal{C}_{\gamma}\}$ forms a $2\gamma$-cover of $\{(f_{\z_j,y^j}(\x)-c_{\z_j,y^j}):y^j\in [0,1]^j\}$ under distance $d_{\hat{\mu}}(f_1,f_2)$ and $|\mathcal{C}'|\le O(\frac{1}{\gamma^{2\mathsf{VC}(\mathcal{F}^{\mathsf{u}})}})$. This implies that the function class $\mathcal{C}''\overset{\text{def}}=\{g(\x',\x)=f(\x')-f(\x):f\in \mathcal{C}',~(\x',\x)\in \mathcal{X}^2\}$ forms a $4\gamma$-cover of $$\mathcal{G}_{\z_j}=\{g_{\z_j,y^j}(\x',\x)=f_{\z_j,y^j}(\x')-f_{\z_j,y^j}(\x):y^j\in [0,1]^j,~(\x',\x)\in \mathcal{X}^2\}$$ under distance $d_{\hat{\nu}}(g_1,g_2)=\mathrm{Pr}_{(\x',\x)\sim\hat{\nu}}[g_1(\x',\x)\not=g_2(\x',\x)]$ for any distribution $\hat{\nu}$ uniform over a fixed \emph{pairing} of $\x^{2N}$ and $|\mathcal{C}''|\le O(\frac{1}{\gamma^{2\mathsf{VC}(\mathcal{F}^{\mathsf{u}})}})$. We have by the chaining argument~\cite[Example 5.24]{wainwright2019high} that $\mathsf{Rad}_N(\mathcal{G}_{\z_j})\le O(\sqrt{\mathsf{VC}(\mathcal{F}^{\mathsf{u}}) N})$. This implies by Lemma~\ref{lem:main2}~\&~\ref{lem:dis2rad} that 
         \begin{equation}
         \label{eq:thm13proof}
             \tilde{r}_{M,N}^{\mathsf{side}}(\mathcal{H},\Phi)\le O\left(\sqrt{\mathsf{VC}(\mathcal{H})M}+\frac{M\sqrt{\mathsf{VC}(\mathcal{F}^{\mathsf{u}})}}{\sqrt{N}}\right).
         \end{equation}
          Taking $M(n)=1.5^n$ in (\ref{eq:eppredictor}), we have $N=S(n)=2\cdot 1.5^n-3$, which ensures $M(n)\le S(n)/2+O(1)$ (as required for (\ref{eq:thm13proof}) to hold). Invoking Lemma~\ref{lem:side2univer}, we conclude
         $$\tilde{r}_T(\mathcal{H},\Psi)\le O(\sqrt{\mathsf{VC}(\mathcal{H})}+\sqrt{\mathsf{VC}(\mathcal{F}^{\mathsf{u}})})\sum_{n=1}^{\lceil\log_{1.5}(T)\rceil}1.5^{n/2}\le O(\sqrt{\mathsf{VC}(\mathcal{F}^{\mathsf{u}})T}),$$
         where the last inequality follows by $\mathcal{H}\subset \mathcal{F}^{\mathsf{u}}$. This completes the proof and the case for $\epsilon_{j+1}=-1$ is symmetric with $\mathcal{F}^{\mathsf{i}}$.
\end{proof}

Note that for the threshold functions $\mathcal{H}=\{1\{x\ge a\}:a,x\in [0,1]\}$ we have $\mathcal{F}^{\mathsf{u}}=\mathcal{F}^{\mathsf{i}}=\mathcal{H}$. Theorem~\ref{cor1} implies an oracle efficient $O(\sqrt{T})$ regret, which matches the information-theoretical lower bound and is tighter than the covering-based $O(\sqrt{T\log T})$ bound implied by~\cite{lazaric2009hybrid}. Another example is the class of indicators of intervals with bounded length $\{1\{x\in [a,b]\}:b-a\ge \gamma, [a,b]\subset [0,1]\}$, for which we have $\mathsf{VC}(\mathcal{F}^{\mathsf{i}})=2$ and $\mathsf{VC}(\mathcal{F}^{\mathsf{u}})\le O(\frac{1}{\gamma})$.

\paragraph{Lipschitz functions.} Let $\mathcal{X}=[0,1]^d$ and $\mathcal{H}\subset [0,1]^{\mathcal{X}}$ be the class of \emph{all} $1$-Lipschitz functions under $L_{\infty}$ norm. Assume $\ell(\hat{y},y)=|\hat{y}-y|$ is the absolute loss. Let $\mu$ and $\hat{\mu}_N$ be the true and empirical distributions, respectively, as in Section~\ref{sec:side}. By Fact~\ref{fact1} and assuming $i.i.d.$ generation of the $\tilde{\x}$s, we have
%    \label{disctodiver}
    $\mathbb{E}_{\x_{-N+1}^j}\sup_{y^j}(\tilde{R}_j-R_j)\le \mathbb{E}_{\x_{-N+1}^0}\sup_{y^j,\z_j}(\mathbb{E}_{\x\sim \mu}[f_{\z_j,y^j}(\x)]-\mathbb{E}_{\x\sim \hat{\mu}_N}[f_{\z_j,y^j}(\x)])$.
By the same argument as Proposition~\ref{prop1} (second part) and Lipschitz property of $h\in\mathcal{H}$, we have:
\begin{fact}
    \label{factlip}
    For all $\z_j,y^j$ and $\x,\x'$, $|f_{\z_j,y^j}(\x)-f_{\z_j,y^j}(\x')|\le 2||\x-\x'||_{\infty}$.
\end{fact}

%We now arrive at our main result for this part:
\begin{theorem}
\label{th-lipschitz}
    Let $\mathcal{H}$ and $\ell$ be as above. Then, there exists an oracle-efficient predictor $\Phi$ such that $\tilde{r}_{T}(\mathcal{H},\Phi)\le \tilde{O}(T^{\max\{\frac{1}{2},\frac{d-1}{d}\}})$, and this bound is tight upto poly-logarithmic factors.
\end{theorem}
\begin{proof}
   By Fact~\ref{factlip}, we know that for all $\z_j,y^j$ the function $f_{\z_j,y^j}(\x)$ is $2$-Lipschitz. Therefore, by Kantorovich-Rubinstein duality~\citep{villani2021topics} we have $\sup_{y^j,\z_j}(\mathbb{E}_{\x\sim \mu}[f_{\z_j,y^j}(\x)]-\mathbb{E}_{\x\sim \hat{\mu}_N}[f_{\z_j,y^j}(\x)])\le 2W_1(\mu,\hat{\mu}_N)$, where $W_1(\mu,\hat{\mu}_N)=\inf_{\gamma\in \Gamma(\mu,\hat{\mu}_N)}\mathbb{E}_{(\x,\x')\sim \gamma}[||\x-\x'||_{\infty}]$ is the Wasserstein $1$-distance with $\Gamma(\mu,\hat{\mu}_N)$ being the class of all coupling between $\mu,\hat{\mu}_N$. Therefore, we have $\mathbb{E}_{\x_{-N+1}^j}\sup_{y^j}(\tilde{R}_j-R_j)\le 2\mathbb{E}_{\x_{-N+1}^0}[W_1(\mu,\hat{\mu}_N)]$, i.e., the discrepancy is upper bounded by the convergence rate of empirical distribution under Wasserstein $1$-distance. Invoking~\cite[Thm 1]{fournier2015rate} and boundedness of $\mathcal{X}$, we have $\mathbb{E}_{\x_{-N+1}^0}[W_1(\mu,\hat{\mu}_N)]\le \tilde{O}(N^{-1/d})$. Let $\Phi$ be the predictor in (\ref{eq:predictor}).  By Lemma~\ref{lem:main2} and $\mathsf{Rad}_M(\mathcal{H})\le\tilde{O}(M^{\max\{\frac{1}{2},\frac{d-1}{d}\}})$~\citep{wainwright2019high}, we have $\tilde{r}_{M,N}^{\mathsf{side}}(\mathcal{H})\le \tilde{O}(M^{\max\{\frac{1}{2},\frac{d-1}{d}\}}+MN^{-1/d})$. The result then follows by Lemma~\ref{lem:side2univer} with $M(n)=2^n$ (which ensures $N=S(n)=M(n)-1$). The last part follows by that the $\epsilon$-metric entropy of $\mathcal{H}$ is $\Theta(\frac{1}{\epsilon^d})$~\citep{wainwright2019high}.
\end{proof}
\begin{remark}
Note that, if we assume certain structure on the distribution $\mu$ that admits a computationally efficient estimator $\hat{\mu}_N$ that satisfies $||\mu-\hat{\mu}_N||_{\mathsf{TV}}\le O(\frac{1}{\sqrt{N}})$ (such as for Gaussian distributions~\citep{ashtiani2018nearly}), then the (optimal) $O(\mathsf{Rad}_T(\mathcal{H})+\sqrt{T})$ bound is achievable for \emph{any} class $\mathcal{H}\subset [0,1]^{\mathcal{X}}$.
\end{remark}
%\begin{remark}
%Note that, if we assume certain structure (e.g., Gaussian) on the \emph{distribution} $\mu$, which allows a computationally efficient estimator $\hat{\mu}_N$ that satisfies $||\mu-\hat{\mu}_N||_{\mathsf{TV}}\le O(\frac{1}{\sqrt{N}})$~\citep{ashtiani2018nearly}, then an $O(\sqrt{\mathsf{VC}(\mathcal{H})T})$ regret bound is achievable for any finite VC class $\mathcal{H}$~\footnote{The total variation bound can be relaxed to the \emph{Wasserstein $1$-distance} if the functions are real valued and Lipschitz.} (see the discussion following Proposition~\ref{prop1}). This provides a different way of beating the $\tilde{O}(T^{\frac{3}{4}})$ barrier.
%\end{remark}

\section{Shifting Distributions}

In this section, we consider a scenario where the underlying feature-generating distribution is allowed to \emph{change} over time. We assume that the total number of changes is upper bounded by $K$, while the selection of distributions and change points are completely arbitrary. It was demonstrated by~\cite{wu2023online} that for finite-VC classes and absolute loss, the regret grows as $O(\sqrt{KT\log T})$ under such feature generation processes~\footnote{The shifting distributions we considered here can be viewed as a special case of the \emph{dynamic changing process} in~\cite{wu2023online}, which bounds the total number of possible distributions instead of the number of changes.}. However, their algorithm depends on the construction of an exponentially sized cover, making it computationally inefficient. We will demonstrate in this section an \emph{oracle-efficient} algorithm that achieves a slightly worse regret.

We first observe that if we \emph{know} the positions of the change points, then we can simply run our oracle-efficient predictor from Theorem~\ref{thm:main1} on each of the segments independently, leading to an $\tilde{O}(T^{\frac{3}{4}}K^{\frac{1}{4}})$ regret. Since there are only $T^K$ possible configurations of the change points, we can therefore run an expert algorithm to aggregate each of such configurations, leading to an $\tilde{O}(\sqrt{KT\log T}+T^{\frac{3}{4}}K^{\frac{1}{4}})$ regret. Unfortunately, this approach has computational cost dominated by $\Omega(T^K)$ and therefore not efficient for large $K$. To address this issue, we instead partition the time horizon into epochs with \emph{fixed} length $B$ and run our oracle-efficient algorithm on each of the epochs independently. The rationale behind this approach is that, if we select $B$ small enough, there will be at least $\frac{T}{B}-K$ epochs with $i.i.d.$ sampling. By tuning the epoch length $B$, we arrive at:

\begin{proposition}
    Let $\mathcal{H}\subset \{0,1\}^{\mathcal{X}}$ be a binary valued class with finite VC-dimension under a convex and $L$-Lipschitz loss. Then there exists an oracle-efficient predictor $\Phi$ such that if the features are generated by the process with change cost $K$ and the labels are generated adversarially, then the hybrid minimax regret as (\ref{eq:hbreg}) is upper bounded by
    $O\left(L\sqrt{\mathsf{VC}(\mathcal{H})\log(LT)}K^{\frac{1}{5}}T^{\frac{4}{5}}\right).$
\end{proposition}
\begin{proof}
    We partition the time horizon into $\frac{T}{B}$ epochs each of length $B$. Let $\Phi$ be the predictor from  Theorem~\ref{thm:main1} that we run on each epochs independently. By the independence, we know that the total regret is upper bounded by the sum of regrets incurred on each of the epochs. Note that, if an epoch does not contain a change point, then the regret is upper bounded by $O(L\sqrt{\mathsf{VC}(\mathcal{H})\log (LT/B)}B^{\frac{3}{4}})$ by Theorem~\ref{thm:main1}, else we  naively upper bounded by $B$. Since there can be at most $K$ epochs containing a change point, the total regret is upper bounded by
    $O\left(\left(\frac{T}{B}-K\right)L\sqrt{\mathsf{VC}(\mathcal{H})\log (LT/B)}B^{\frac{3}{4}}+KB\right)$ by
    optimizing over $B$: We find that $B=T^{\frac{4}{5}}K^{-\frac{4}{5}}$ attains the minimum and the result follows.
\end{proof}
%\begin{remark}
%    Note that if we can achieve a regret of $O(T^{r})$ in the $i.i.d.$ case for some $r \geq 0$, then the regret for the changing distribution with a change cost of $K$ is upper bounded by $O(T^{\frac{1}{2-r}}K^{\frac{r}{2-r}})$. Specifically, for a class with an $\alpha$-fat shattering dimension of order $p > 0$, we obtain a regret of $\tilde{O}(T^{\frac{p+2}{p+3}}K^{\frac{1}{p+3}})$. %Moreover generally, one may also consider other types of changing processes such as when the distribution at each time step is within an $\epsilon$-total variation ball of an unknown reference distributions.
%\end{remark}

\section{Contextual $K$-arm Bandits}

We now briefly discuss the extension to the contextual $K$-arm bandits as introduced in Section~\ref{sec:def}, leaving details to Appendix~\ref{sec:appbandit}. The basic idea follows the same path as in the online learning case, where we partition the time step into epochs, at each epoch we use the sample observed thus far to estimate the underlying distribution and use the estimated distribution to generate the \emph{hallucinated samples} as needed in the relaxation based algorithms. 

\paragraph{Bandit predictor with side-information.}Let $(\x_1,c_1),\cdots,(\x_M,c_M)$ be any realization of the context-cost pairs and $\x_{-N+1}^0$ be the side-information with $\x_{-N+1}^M$ sampled $i.i.d.$ from an (unknown) distribution $\mu$. At each time step $j\in [M]$, we consider the following prediction rule adapted from~\cite{syrgkanis2016improved} by generating the hallucinated samples from $\hat{\mu}_N$:
\begin{itemize}
    \item[1.] Sample $\tilde{\x}_{j+1}^M$ from $\hat{\mu}_N$ \emph{without replacement}; $\epsilon_{j+1}^M$ $i.i.d.$ from uniform distribution over $\{\pm 1\}^K$; $Z_{j+1}^M$ $i.i.d.$ from distribution over $\{0,\frac{1}{\gamma}\}$ such that $\mathrm{Pr}[Z_i=\frac{1}{\gamma}]=\gamma K$, where $\gamma$ is a parameter to be tuned; Let $e_k$ be the standard base of $\mathbb{R}^K$ with coordinate $k$ being $1$;
    \item[2.] Let $\mathcal{D}_K$ be the class of distributions over $[K]$.  Find (using ERM oracle):
    \begin{equation*}
    \scalebox{0.92}{% 
    $\begin{aligned}
    \hat{q}_j=\mathop{\mathrm{arg\,min}}\limits_{q\in \mathcal{D}_K}\sup_{p_j\in \mathcal{D}'}\mathbb{E}_{\hat{c}_j\sim p_j}\left\{\langle q,\hat{c}_j\rangle-\inf_{h\in \mathcal{H}}\left(\sum_{i=1}^{j}\hat{c}_i[h(\x_i)]+\sum_{i={j+1}}^M2\epsilon_i[h(\tilde{\x}_i)]Z_i\right)+\gamma (M-j)K\right\},
    \end{aligned}$}
    \end{equation*}
    where $\mathcal{D}'$ is the class of distributions over $\{\frac{1}{\gamma} e_k:k\in [K]\}\cup \{\mathbf{0}\}$ such that the probability equals $\frac{1}{\gamma}e_k$ is upper bounded by $\gamma$ for all $k\in [K]$; and $\mathbf{0}\in \mathbb{R}^K$ is the all $0$s vector;
    \item[3.] Define $q_j=(1-\gamma K)\hat{q}_j+\gamma \mathbf{1}$ and make prediction $\hat{y}_j\sim q_j$, where $\mathbf{1}$ is the all $1$s vector.
\end{itemize}
Here, the $\hat{c}_i$ at step $2$ is an unbiased estimation of $c_i$ as $\hat{c}_i=\frac{1}{\gamma} I_i e_{\hat{y}_i}$, where $I_i$ is the indicator that takes value $1$ w.p. $\frac{\gamma c_i[\hat{y}_i]}{q_i[\hat{y}_i]}$ with $\hat{y}_i$ and $q_i$ being the predictions at step $i\le j-1$.

\paragraph{Analysis of regret.} Our key idea is to define a \emph{surrogate} relaxation $R_j$ and $\tilde{R}_j$ for the predictor $\hat{q}_j$ and establish a bandit version of decomposition for the regret with side-information as in Lemma~\ref{lem:main2}. This is achieved via the concept of \emph{approx-admissibility} as in Lemma~\ref{lem:admiss} and a careful adaption of the admissibility proof from~\cite{syrgkanis2016improved}. The technical challenge then boils down to bounding the discrepancies between \(R_j\) and \(\tilde{R}_j\), as in Lemma~\ref{lem:main2}. To this end, we employ a technique similar to Lemma~\ref{lem:dis2rad} that relates the discrepancies to a Rademacher sum and show that the \emph{sensitivity} of the functions in the sum is upper bounded by \(O(K)\). Crucially, the sensitivity is \emph{independent} of $\gamma$, and therefore the regret with side information is optimized at $\gamma=(\log|\mathcal{H}|/KM)^{\frac{1}{3}}$.  By leveraging a similar epoch approach as in the proof of Theorem~\ref{thm:main1}, and setting the epoch length to \(n^{\frac{3}{2}}\), we arrive at the main result of this section (see Appendix~\ref{sec:appbandit} for a detailed  proof): %(with detailed proof deferred to Appendix~\ref{sec:appbandit}):

%we will leverage the relaxation introduced in~\cite{syrgkanis2016improved}, which provides an oracle-efficient algorithm that achieves a regret upper bounded by \(O((TK)^{\frac{2}{3}}(\log|\mathcal{H}|)^{\frac{1}{3}})\) for a finite policy set \(\mathcal{H}\) with \emph{known} i.i.d. processes. Our main idea is to replace the hallucinated features used in the relaxation with the ones generated by the \emph{empirical distribution} of samples from prior epochs, to form the \emph{surrogate} relaxations as in (\ref{eq:relax}) and (\ref{eq:relaxshit}). We then establish a bandit version of decomposition for the regret with side-information, akin to Lemma~\ref{lem:main2}. The critical technical justification allowing this approach is that the admissibility proof in~\cite{syrgkanis2016improved} is \emph{oblivious} to the distribution of the hallucinated features \(\tilde{\x}\)s. The key technical challenge then boils down to bounding the discrepancies between \(R_j\) and \(\tilde{R}_j\), as in Lemma~\ref{lem:main2}. To address this issue, we employ a technique similar to Lemma~\ref{lem:dis2rad} that relates the discrepancies to a Rademacher sum and showing that the \emph{sensitivity} of the functions in the sum is upper bounded by \(O(K)\). By adopting a similar epoch approach as in the proof of Theorem~\ref{thm:main1}, and setting the epoch length to \(n^{\frac{3}{2}}\), we arrive at the main result of this section (with detailed proof deferred to Appendix~\ref{sec:appbandit}):

\begin{theorem}
    Let $\mathcal{H}\subset [K]^{\mathcal{X}}$ be a finite policy set. Then there exists an oracle-efficient predictor $\Phi$ such that the hybrid bandit minimax regret defined in Section~\ref{sec:def} is upper bounded by
    $$\tilde{r}_T^{\mathsf{bandit}}(\mathcal{H},\Phi)\le O\left((K^{\frac{2}{3}}(\log|\mathcal{H}|)^{\frac{1}{3}}+K\sqrt{\log K})\cdot T^{\frac{4}{5}}\right).$$
\end{theorem}
\begin{remark}
    Note that, our primary focus here is on the dependency on $T$. We believe a more careful selection of the epoch length and employing techniques from~\cite{banihashem2023an} could result in a better dependency on $K$. We leave it as an open problem to improve the $\frac{4}{5}$ exponent of $T$.
\end{remark}

% Acknowledgments---Will not appear in anonymized version
%\acks{We thank a bunch of people and funding agency.}

\bibliography{ml.bib}

\appendix

\newpage
\section{Proof of Lemma~\ref{lem:admiss}}
\label{sec:appproflem5}
In this section, we establish the \emph{approx-admissibility} of our predictor in (\ref{eq:predictor}). The reasoning follows closely to the arguments as in~\cite[Lemma 11\&12]{rakhlin2012relax} but needs careful adaption for handling the hallucinated samples $\tilde{\x}$s generated from $\hat{\mu}_N$. We have
    \begin{align*}
    \mathbb{E}_{\x_j}\sup_{y_j}\mathbb{E}_{\epsilon,\tilde{\x}}\left[\ell(\hat{y}_j,y_j)+R_j\right]&\overset{(a)}{=} \mathbb{E}_{\x_j}\sup_{y_j}\mathbb{E}_{\epsilon,\tilde{\x}}\left[\ell(\hat{y}_j,y_j)+\sup_{h\in \mathcal{H}}2L\sum_{i=j+1}^M\epsilon_ih(\tilde{\x}_i)-L_j^h\right]\\
        &\le \mathbb{E}_{\epsilon,\tilde{\x}}\mathbb{E}_{\x_j}\left[\sup_{y_j} \ell(\hat{y}_j,y_j)+\sup_{h\in \mathcal{H}}2L\sum_{i=j+1}^M\epsilon_ih(\tilde{\x}_i)-\ell(h(\x_j),y_j)-L_{j-1}^h\right]\\
        &\overset{(b)}{\le} \mathbb{E}_{\epsilon,\tilde{\x}}\mathbb{E}_{\x_j}\left[\inf_{\hat{y}}\sup_{y_j} \ell(\hat{y},y_j)+\sup_{h\in \mathcal{H}}2L\sum_{i=j+1}^M\epsilon_ih(\tilde{\x}_i)-\ell(h(\x_j),y_j)-L_{j-1}^h\right]\\
        &=\mathbb{E}_{\epsilon,\tilde{\x}}\mathbb{E}_{\x_j}\left[\inf_{\hat{y}}\sup_{y_j} \sup_{h\in \mathcal{H}}2L\sum_{i=j+1}^M\epsilon_ih(\tilde{\x}_i)-L_{j-1}^h+\ell(\hat{y},y_j)-\ell(h(\x_j),y_j)\right]\\
        &\overset{(c)}{\le} \mathbb{E}_{\epsilon,\tilde{\x}}\mathbb{E}_{\x_j}\left[\inf_{\hat{y}}\sup_{y_j} \sup_{h\in \mathcal{H}}2L\sum_{i=j+1}^M\epsilon_ih(\tilde{\x}_i)-L_{j-1}^h+\partial \ell(\hat{y},y_j)(\hat{y}-h(\x_j))\right]\\
        &\overset{(d)}{\le} \mathbb{E}_{\epsilon,\tilde{\x}}\mathbb{E}_{\x_j}\left[\inf_{\hat{y}}\sup_{y_j}\sup_{g_j\in [-L,L]} \sup_{h\in \mathcal{H}}2L\sum_{i=j+1}^M\epsilon_ih(\tilde{\x}_i)-L_{j-1}^h+ g_j(\hat{y}-h(\x_j))\right]\\
        &\overset{(e)}{\le} \mathbb{E}_{\epsilon,\tilde{\x}}\mathbb{E}_{\x_j}\left[\inf_{\hat{y}}\sup_{g_j\in \{-L,L\}} \sup_{h\in \mathcal{H}}2L\sum_{i=j+1}^M\epsilon_ih(\tilde{\x}_i)-L_{j-1}^h+g_j(\hat{y}-h(\x_j))\right]
    \end{align*}
    where $(a)$ follows by the definition of $R_j$ and that $\hat{y}_j$ has the same randomness as $R_j$ (i.e, the $\tilde{\x}$s and $\epsilon$s); $(b)$ is due to definition of $\hat{y}_j$; $(c)$ is due to convexity of $\ell$; $(d)$ is due to $L$-Lipschitz property of $\ell$; $(e)$ follows by that the inner function is convex w.r.t. $g_j$ and thus the $\sup_{g_j\in [-L,L]}$ is attained on the boundary $\{-L,L\}$. We have
    \begin{align*}
    \mathbb{E}_{\epsilon,\tilde{\x}}&\mathbb{E}_{\x_j}\left[\inf_{\hat{y}}\sup_{g_j\in \{-L,L\}} \sup_{h\in \mathcal{H}}2L\sum_{i=j+1}^M\epsilon_ih(\tilde{\x}_i)-L_{j-1}^h+g_j(\hat{y}-h(\x_j))\right]\\
        &\overset{(a)}{=}  \mathbb{E}_{\epsilon,\tilde{\x}}\mathbb{E}_{\x_j}\left[\inf_{\hat{y}}\sup_{d_j\in \Delta(\{-L,L\})}\mathbb{E}_{g_j\sim d_j}\left[\sup_{h\in \mathcal{H}}2L\sum_{i=j+1}^M\epsilon_ih(\tilde{\x}_i)-L_{j-1}^h+g_j(\hat{y}-h(\x_j))\right]\right]\\
        &\overset{(b)}{=}\mathbb{E}_{\epsilon,\tilde{\x}}\mathbb{E}_{\x_j}\left[\sup_{d_j\in \Delta(\{-L,L\})}\inf_{\hat{y}}\mathbb{E}_{g_j\sim d_j}\left[\sup_{h\in \mathcal{H}}2L\sum_{i=j+1}^M\epsilon_ih(\tilde{\x}_i)-L_{j-1}^h+g_j(\hat{y}-h(\x_j))\right]\right]\\
        &\overset{(c)}{=}\mathbb{E}_{\epsilon,\tilde{\x}}\mathbb{E}_{\x_j}\left[\sup_{d_j}\inf_{\hat{y}}\mathbb{E}_{g_j\sim d_j}\left[g_j\hat{y}+\sup_{h\in \mathcal{H}}2L\sum_{i=j+1}^M\epsilon_ih(\tilde{\x}_i)-L_{j-1}^h-g_jh(\x_j)\right]\right]\\
        &=\mathbb{E}_{\epsilon,\tilde{\x}}\mathbb{E}_{\x_j}\left[\sup_{d_j}\inf_{\hat{y}}\left(\mathbb{E}_{g_j\sim d_j}[g_j\hat{y}]+\mathbb{E}_{g_j\sim d_j}\left[\sup_{h\in \mathcal{H}}2L\sum_{i=j+1}^M\epsilon_ih(\tilde{\x}_i)-L_{j-1}^h-g_jh(\x_j)\right]\right)\right]\\
        &\overset{(d)}{=}\mathbb{E}_{\epsilon,\tilde{\x}}\mathbb{E}_{\x_j}\left[\sup_{d_j}\left(\inf_{\hat{y}}\mathbb{E}_{g'_j\sim d_j}[g_j'\hat{y}]\right)+\mathbb{E}_{g_j\sim d_j}\left[\sup_{h\in \mathcal{H}}2L\sum_{i=j+1}^M\epsilon_ih(\tilde{\x}_i)-L_{j-1}^h-g_jh(\x_j)\right]\right]\\
         &=\mathbb{E}_{\epsilon,\tilde{\x}}\mathbb{E}_{\x_j}\left[\sup_{d_j}\mathbb{E}_{g_j\sim d_j}\left[\inf_{\hat{y}}\mathbb{E}_{g'_j\sim d_j}[g_j'\hat{y}]+\sup_{h\in \mathcal{H}}2L\sum_{i=j+1}^M\epsilon_ih(\tilde{\x}_i)-L_{j-1}^h-g_jh(\x_j)\right]\right]\\&=\mathbb{E}_{\epsilon,\tilde{\x}}\mathbb{E}_{\x_j}\left[\sup_{d_j}\mathbb{E}_{g_j\sim d_j}\left[\sup_{h\in \mathcal{H}}2L\sum_{i=j+1}^M\epsilon_ih(\tilde{\x}_i)-L_{j-1}^h+\inf_{\hat{y}}\mathbb{E}_{g'_j\sim d_j}[g_j'\hat{y}]-g_jh(\x_j)\right]\right]\\
        &\overset{(e)}{\le} \mathbb{E}_{\epsilon,\tilde{\x}}\mathbb{E}_{\x_j}\left[\sup_{d_j}\mathbb{E}_{g_j\sim d_j}\left[\sup_{h\in \mathcal{H}}2L\sum_{i=j+1}^M\epsilon_ih(\tilde{\x}_i)-L_{j-1}^h+\mathbb{E}_{g'_j\sim d_j}[g_j'h(\x_j)]-g_jh(\x_j)\right]\right]\\
        &\overset{(f)}{\le} \mathbb{E}_{\epsilon,\tilde{\x}}\mathbb{E}_{\x_j}\left[\sup_{d_j}\mathbb{E}_{g_j,g_j'\sim d_j}\left[\sup_{h\in \mathcal{H}}2L\sum_{i=j+1}^M\epsilon_ih(\tilde{\x}_i)-L_{j-1}^h+(g'_j-g_j)h(\x_j)\right]\right]\\
        &\overset{(g)}{=}\mathbb{E}_{\epsilon,\tilde{\x}}\mathbb{E}_{\x_j}\left[\sup_{d_j}\mathbb{E}_{g_j,g_j'\sim d_j}\mathbb{E}_{\epsilon_j}\left[\sup_{h\in \mathcal{H}}2L\sum_{i=j+1}^M\epsilon_ih(\tilde{\x}_i)-L_{j-1}^h+\epsilon_j(g'_j-g_j)h(\x_j)\right]\right]\\
        &=\mathbb{E}_{\epsilon,\tilde{\x}}\mathbb{E}_{\x_j}\left[\sup_{d_j}\mathbb{E}_{g_j,g_j'\sim d_j}\mathbb{E}_{\epsilon_j}\left[\sup_{h\in \mathcal{H}}\underbrace{\left(2L\sum_{i=j+1}^M\epsilon_ih(\tilde{\x}_i)-L_{j-1}^h\right)}_{A}+\underbrace{\epsilon_j g'_j h(\x_j)}_{B}+\underbrace{(-\epsilon_jg_jh(\x_j))}_{C}\right]\right]\\
        &\overset{(h)}{\le} \mathbb{E}_{\epsilon,\tilde{\x}}\mathbb{E}_{\x_j}\left[\sup_{d_j}\mathbb{E}_{g_j\sim d_j}\mathbb{E}_{\epsilon_j}\left[\sup_{h\in \mathcal{H}}2L\sum_{i=j+1}^M\epsilon_ih(\tilde{\x}_i)-L_{j-1}^h+2\epsilon_j g_j h(\x_j)\right]\right]\\
        &\overset{(i)}{=}\mathbb{E}_{\epsilon,\tilde{\x}}\mathbb{E}_{\x_j}\left[\mathbb{E}_{\epsilon_j}\left[\sup_{h\in \mathcal{H}}2L\sum_{i=j+1}^M\epsilon_ih(\tilde{\x}_i)-L_{j-1}^h+2\epsilon_j L h(\x_j)\right]\right]\\
        &=\tilde{R}_{t-1},
    \end{align*}
    where $(a)$ follows by $\sup_{g_j\in \{-L,L\}}\equiv\sup_{d_j\in \Delta(\{-L,L\})}\mathbb{E}_{g_j\sim d_j}$ where $\Delta(\{-L,L\})$ is the set of all probability distributions over $\{-L,L\}$; $(b)$ follows by the minimax theorem and noticing that the inner expectation is bi-linear w.r.t. $\hat{y}$ and $d_j$; $(c)$ follows by the fact that $g_j\hat{y}$ is independent of $\sup_h$; $(d)$ follows by that the $\sup_h$ term is independent of $\hat{y}$ and introducing an $i.i.d.$ copy $g_j'$ of $g_j$; $(e)$ follows by the fact that replacing $\hat{y}$ with $h(\x_j)$ does not decrease the $\inf$ term; $(f)$ is due to $\sup\mathbb{E}\le \mathbb{E}\sup$; $(g)$ is due to symmetries of $g_j,g_j'$ and $\epsilon_j$ is uniform over $\{-1,1\}$; $(h)$ follows by $\sup(A+B+C)\le \sup (A/2+B)+\sup(A/2+C)=(\sup(A+2B)+\sup(A+2C))/2$, the linearity of expectation and symmetries of $B,C$; $(i)$ follows by that the inner expectation takes the same value for all $g_j\in \{-L,L\}$ and therefore the $\sup_{d_j}\mathbb{E}_{g_j\sim d_j}$ can be eliminated. This completes the proof.

\section{Proof of Lemma~\ref{lem:main2}}
\label{sec:prooflem6}
     Denote $\mathbb{Q}_j\equiv \mathbb{E}_{\x_{-N+1}^0}\mathbb{E}_{\x_1}\sup_{y_1}\mathbb{E}_{\hat{y}_1}\cdots \mathbb{E}_{\x_j}\sup_{y_j}\mathbb{E}_{\hat{y}_j}$ for notation convenience. We have
    {\fontsize{10pt}{12pt}\selectfont\begin{align*}
    \tilde{r}_{M,N}^{\mathsf{side}}(\mathcal{H},\Phi)
    &\overset{(a)}{=} \mathbb{Q}_M\left[\sum_{j=1}^M\ell(\hat{y}_j,y_j)+R_M\right]\\
    &\overset{(b)}{=} \mathbb{Q}_{M-1}\left[\sum_{j=1}^{M-1}\ell(\hat{y}_j,y_j)+\mathbb{E}_{\x_{M}}\sup_{y_{M}}\mathbb{E}_{\hat{y}_{M}}\left[\ell(\hat{y}_M,y_M)+R_M\right]\right]\\
    &\overset{(c)}{\le} \mathbb{Q}_{M-1}\left[\sum_{j=1}^{M-1}\ell(\hat{y}_j,y_j)+\tilde{R}_{M-1}\right]\\
    &= \mathbb{Q}_{M-1}\left[\sum_{j=1}^{M-1}\ell(\hat{y}_j,y_j)+R_{M-1}+(\tilde{R}_{M-1}-R_{M-1})\right]\\
    &\overset{(d)}{\le} \mathbb{Q}_{M-1}\left[\sum_{j=1}^{M-1}\ell(\hat{y}_j,y_j)+R_{M-1}\right]+\mathbb{E}_{\x_{-N+1}^0}\mathbb{E}_{\x^{M-1}}\sup_{y^{M-1}}(\tilde{R}_{M-1}-R_{M-1})\\
    &\overset{(e)}{\le} \mathbb{E}_{\x_{-N+1}^0}[\tilde{R}_0]+\sum_{j=1}^{M-1}\mathbb{E}_{\x_{-N+1}^0}\mathbb{E}_{\x^j}\sup_{y^j}(\tilde{R}_j-R_j),
    \end{align*}}where $(a)$ follows by definition of $R_M$; $(b)$ follows by extracting the last layer of $\mathbb{Q}_M$; $(c)$ follows by Lemma~\ref{lem:admiss} and noticing that $\hat{y}_j$ has the same randomness as $R_j$; $(d)$ follows by the the facts that $\sup (A+B)\le \sup A+\sup B$, $\sup\mathbb{E}\le \mathbb{E}\sup$, the linearity of expectation and $\tilde{R}_{M-1}-R_{M-1}$ is independent of $\hat{y}_j$ for all $j\le M-1$; $(e)$ follows by repeating the same arguments for another $M-1$ steps. This completes the proof.

\section{Proof of Lemma~\ref{lem:dis2rad}}
\label{sec:applem8}
We have
    \begin{align*}
    \mathbb{E}_{\x_{-N+1}^0}\mathbb{E}_{\x^j}\sup_{y^j}(\tilde{R}_t-R_j)&\overset{(a)}{=} \mathbb{E}_{\x_{-N+1}^0}\mathbb{E}_{\x^j}\sup_{y^j} \mathbb{E}_{\tilde{\x}_{j+1}^M,\epsilon_{j+1}^M}\mathbb{E}_{\x\sim \mu}[f_{\z_j,y^j}(\x)-f_{\z_j,y^j}(\tilde{\x}_{j+1})]\\
    &\le \mathbb{E}_{\x_{-N+1}^0}\mathbb{E}_{\x^j}\mathbb{E}_{\tilde{\x}_{j+2}^M,\epsilon_{j+1}^M}\sup_{y^j}\mathbb{E}_{\tilde{\x}_{j+1},\x\sim \mu}[f_{\z_j,y^j}(\x)-f_{\z_j,y^j}(\tilde{\x}_{j+1})]\\
    &\overset{(b)}{=}\mathbb{E}_{\x_{-N+1}^0}\mathbb{E}_{\z_j}\sup_{y^j}\mathbb{E}_{\tilde{\x}_{j+1},\x\sim \mu}[f_{\z_j,y^j}(\x)-f_{\z_j,y^j}(\tilde{\x}_{j+1})]\\
    &=\mathbb{E}_{\x_{-N+1}^0}\mathbb{E}_{\z_j}\sup_{y^j}[\mathbb{E}_{\x\sim \mu}[f_{\z_j,y^j}(\x)]-\mathbb{E}_{\tilde{\x}_{j+1}}[f_{\z_j,y^j}(\tilde{\x}_{j+1})]]\\
    &\overset{(c)}{=} \mathbb{E}_{\z_j}\mathbb{E}_{\x_{-N+1}^{-N+B}}\sup_{y^j}\left[\mathbb{E}_{\x\sim \mu}[f_{\z_j,y^j}(\x)]-\frac{1}{B}\sum_{i=1}^{B}f_{\z_j,y^j}(\x_{-N+i})\right]\\
    &\overset{(d)}{=}\mathbb{E}_{\z_j}\mathbb{E}_{\x_{-N+1}^{-N+B}}\sup_{y^j}\left[\frac{1}{B}\sum_{i=1}^{B}\mathbb{E}_{\x_i'\sim \mu}[f_{\z_j,y^j}(\x_i')]-f_{\z_j,y^j}(\x_{-N+i})\right]\\
    &=\mathbb{E}_{\z_j}\mathbb{E}_{\x_{-N+1}^{-N+B}}\sup_{y^j}\mathbb{E}_{\x'^B\sim \mu^{\otimes B}}\left[\frac{1}{B}\sum_{i=1}^Bf_{\z_j,y^j}(\x_i')-f_{\z_j,y^j}(\x_{-N+i})\right]\\
    &\le \mathbb{E}_{\z_j}\mathbb{E}_{\x_{-N+1}^{-N+B}}\mathbb{E}_{\x'^B}\sup_{y^j}\left[\frac{1}{B}\sum_{i=1}^Bf_{\z_j,y^j}(\x_i')-f_{\z_j,y^j}(\x_{-N+i})\right]\\
    &\overset{(e)}{=}\mathbb{E}_{\z_j}\mathbb{E}_{\x_{-N+1}^{-N+B}}\mathbb{E}_{\x'^B}\mathbb{E}_{\epsilon'^B}\sup_{y^j}\left[\frac{1}{B}\sum_{i=1}^B\epsilon_j'(f_{\z_j,y^j}(\x_i')-f_{\z_j,y^j}(\x_{-N+i}))\right]\\
    &\le \sup_{\x_{-N+1}^{-N+B},\x'^B,\z_j}\mathbb{E}_{\epsilon'^B}\sup_{y^j}\left[\frac{1}{B}\sum_{i=1}^B\epsilon_j'(f_{\z_j,y^j}(\x_i')-f_{\z_j,y^j}(\x_{-N+i}))\right]
\end{align*}
where $(a)$ follows by Fact~\ref{fact1} (in Section~\ref{sec:side}); $(b)$ follows by definition of $\z_j$; $(c)$ follows by Fact~\ref{fact3} below and taking $B=N-M+j+1$; $(d)$ follows by introducing $B$ fresh $i.i.d.$ samples $\x'^{B}\sim \mu^{\otimes B}$; $(e)$ follows by symmetries of $\x'^B$ and $\x_{-N+1}^{-N+B}$ (which are independent of $\z_j$) and introducing the $i.i.d.$ random variables $\epsilon'^B$ uniform over $\{-1,1\}^B$;

\begin{fact}
    \label{fact3}
    Let $B=N-M+j+1$, then
    \begin{align*}
        \mathbb{E}_{\x_{-N+1}^0}\mathbb{E}_{\z_j}\sup_{y^j}[\mathbb{E}_{\x\sim \mu}[f_{\z_j,y^j}(\x)]&-\mathbb{E}_{\tilde{\x}_{j+1}}[f_{\z_j,y^j}(\tilde{\x}_{j+1})]]\\&= \mathbb{E}_{\z_j}\mathbb{E}_{\x_{-N+1}^{-N+B}}\sup_{y^j}\left[\mathbb{E}_{\x\sim \mu}[f_{\z_j,y^j}(\x)]-\frac{1}{B}\sum_{i=1}^{B}f_{\z_j,y^j}(\x_{-N+i})\right].
    \end{align*}
\end{fact}
\begin{proof}
    Note that $\z_j=(\x^j,\tilde{\x}_{j+2}^M,\epsilon_{j+1}^M)$, where $\tilde{\x}_{j+1}^M$ are sampled uniformly from $\x_{-N+1}^0$ \emph{without replacement}, and $\x^j$, $\epsilon_{j+1}^M$ are independent of $\x_{-N+1}^0$. Therefore, we have
    \begin{align*}
        \mathbb{E}_{\x_{-N+1}^0}\mathbb{E}_{\z_j}\sup_{y^j}[&\mathbb{E}_{\x\sim \mu}[f_{\z_j,y^j}(\x)]-\mathbb{E}_{\tilde{\x}_{j+1}}[f_{\z_j,y^j}(\tilde{\x}_{j+1})]]\\&=\mathbb{E}_{\x^j,\epsilon_{j+1}^M}\mathbb{E}_{\x_{-N+1}^0}\mathbb{E}_{\tilde{\x}_{j+2}^M}\sup_{y^j}[\mathbb{E}_{\x\sim \mu}[f_{\z_j,y^j}(\x)]-\mathbb{E}_{\tilde{\x}_{j+1}}[f_{\z_j,y^j}(\tilde{\x}_{j+1})]]\\
        &\overset{(\star)}{=}\mathbb{E}_{\x^j,\epsilon_{j+1}^M}\mathbb{E}_{\x_{-N+1}^0}\mathbb{E}_I\sup_{y^j}\left[\mathbb{E}_{\x\sim \mu}[f_{\z_j,y^j}(\x)]-\frac{1}{B}\sum_{i\in [N]\backslash I}f_{\z_j,y^j}(\x_{-N+i})\right]
    \end{align*}
    where the key step $(\star)$ follows by noticing that the randomness of $\tilde{\x}_{j+2}^M$ is equivalent to selecting a \emph{random} index set $I\subset [N]$ uniformly with size $|I|=M-j-1$ and the index of $\tilde{\x}_{j+1}$ (in $\x_{-N+1}^0$) is then uniform over $[N]\backslash I$~\footnote{By the definition of sampling \emph{without replacement}.}, where the size of $[N]\backslash I$ is $B=N-M+j+1$; Therefore, $$\mathbb{E}_{\tilde{\x}_{j+1}}[f_{\z_j,y^j}(\tilde{\x}_{j+1})]=\frac{1}{B}\sum_{i\in [N]\backslash I}f_{\z_j,y^j}(\x_{-N+i}).$$

    Note that $\x_{-N+1}^0$ is an $i.i.d.$ sample, by \emph{symmetries}, we can \emph{fix} $I=\{B+1,\cdots,N\}$ (i.e., we take $\tilde{\x}_{j+2}^M$ being $\x_{-N+B+1}^0$) and therefore $\tilde{\x}_{j+2}^M$ can be decoupled from $\x_{-N+1}^{-N+B}$, leading to
    \begin{align*}
        \mathbb{E}_{\x^j,\epsilon_{j+1}^M}\mathbb{E}_{\x_{-N+1}^0}\mathbb{E}_I\sup_{y^j}&\left[\mathbb{E}_{\x\sim \mu}[f_{\z_j,y^j}(\x)]-\frac{1}{B}\sum_{i\in [N]\backslash I}f_{\z_j,y^j}(\x_{-N+i})\right]\\
        &=\mathbb{E}_{\x^j,\epsilon_{j+1}^M}\mathbb{E}_{\tilde{\x}_{j+2}^M}\mathbb{E}_{\x_{-N+1}^{-N+B}}\sup_{y^j}\left[\mathbb{E}_{\x\sim \mu}[f_{\z_j,y^j}(\x)]-\frac{1}{B}\sum_{i=1}^Bf_{\z_j,y^j}(\x_{-N+i})\right]\\
        &=\mathbb{E}_{\z_j}\mathbb{E}_{\x_{-N+1}^{-N+B}}\sup_{y^j}\left[\mathbb{E}_{\x\sim \mu}[f_{\z_j,y^j}(\x)]-\frac{1}{B}\sum_{i=1}^{B}f_{\z_j,y^j}(\x_{-N+i})\right].
    \end{align*}
    This completes the proof of the Fact.
\end{proof}

\section{Omitted Proofs}
\label{sec:omit}
In this section, we collect all other proofs that are omitted from the main text.
\begin{proof}[Proof of Proposition~\ref{prop1}]
      Denote $$F(h)=2L\sum_{i=j+2}^M\epsilon_i h(\tilde{\x}_i)-L_j^h.$$ Let $\hat{h}=\arg\max_{h\in \mathcal{H}}F(h)$ (find an approximation if necessary). We claim that for any $\x\in \mathcal{X}$, $$F(\hat{h})-2L\le \sup_{h\in \mathcal{H}} \left\{2\epsilon_{j+1}Lh(\x)+F(h)\right\}\le F(\hat{h})+2L.$$ This will complete the proof of the first part. To see the upper bound, we have
    $$\sup_{h\in \mathcal{H}}\left\{2\epsilon_{j+1} Lh(\x)+F(h)\right\}\le \sup_h \{2\epsilon_{j+1} Lh(\x)\}+\sup_h F(h)\le 2L+F(\hat{h}),$$ since $h(\x)\in [0,1]$. For the lower bound, we have
    $$\sup_{h\in \mathcal{H}}\{2\epsilon_{j+1} Lh(\x)+F(h)\}\ge 2\epsilon_{j+1} L\hat{h}(\x)+F(\hat{h})\ge F(\hat{h})-2L,$$ since $\sup$ do not increase by replacing $h$ with any specific $\hat{h}$ and $\hat{h}(\x)\in [0,1]$.

    To prove the second part, for any given $h\in \mathcal{H}$, we denote $$g_h(y^j)=2L\epsilon_{j+1}h(\x)+2L\sum_{i=j+2}^M\epsilon_i h(\tilde{\x}_i)-L_j^h.$$ Note that, $y^j$ only appears in the $L_j^h$ term. By definition of $L_j^h$ and $L$-Lipschitz property of the loss $\ell$, we have $$\forall h\in \mathcal{H},~|g_h(y^j)-g_h(y'^j)|\le jL||y^j-y'^j||_{\infty}.$$ Let $\hat{h}=\arg\max_{h} g_h(y^j)$, we have $$\sup_h g_h(y^j)-\sup_h g_h(y'^j)\le g_{\hat{h}}(y^j)-g_{\hat{h}}(y'^j)\le jL||y^j-y'^j||_{\infty}.$$ Let $\hat{h}'=\arg\max_{h} g_h(y'^j)$, we have $$\sup_h g_h(y^j)-\sup_h g_h(y'^j)\ge g_{\hat{h}'}(y^j)-g_{\hat{h}'}(y'^j)\ge -jL||y^j-y'^j||_{\infty}.$$ The proposition follows by noticing that $f_{\z_j,y^j}(\x)=\sup_h g_h(y^j)$.
\end{proof}

    \begin{proof}[Proof of Fact~\ref{fact2}]
        Note that $f_{\z_j,y^j}(\x)=\sup_{h} \{2h(\x)+F(h)\}$. If $h^0(\x)=1$, then $\exists h\in \mathcal{H}^0$ such that $h(\x)=1$ and $F(h)=F(\hat{h})$, thus $f_{\z_j,y^j}(\x)\ge 2+F(\hat{h})$. Clearly, we also have $f_{\z_j,y^j}(\x)\le 2\sup_h h(\x)+\sup_h F(h)\le 2+F(\hat{h})$, the first case follows. If $h^0(\x)=0$ and $h^1(\x)=1$, then there exists $h\in \mathcal{H}^1$ such that $h(\x)=1$ and $F(h)=F(\hat{h})-1$, thus $f_{\z_j,y^j}(\x)\ge F(\hat{h})-1+2=F(\hat{h})+1$. On the other-hand, since $h^0(\x)=0$, we have for all $h\in \mathcal{H}^0$, $2h(\x)+F(h)=F(\hat{h})$. For any other $h\not\in \mathcal{H}^0\cup\mathcal{H}^1$, we have $2h(\x)+F(h)\le 2+F(\hat{h})-2=F(\hat{h})$. Therefore, $f_{\z_j,y^j}(\x)\le F(\hat{h})+1$, this completes the second case. Finally, if both $h^0(\x)=h^1(\x)=0$, we have for any $h\in \mathcal{H}^0$, $2h(\x)+F(h)= F(\hat{h})$, i.e., $f_{\z_j,y^j}(\x)\ge F(\hat{h})$. Moreover, for any $h\not\in \mathcal{H}^0$, it is easy to verify that $2h(\x)+F(h)\le F(\hat{h})$. This completes the proof.
    \end{proof}

\section{Analysis of Contextual $K$-arm Bandits}
\label{sec:appbandit}

In this appendix, we provide a detailed analysis of the contextual $K$-arm bandit problem with the contexts generated by an \emph{unknown} $i.i.d.$ process and losses generated adversarially, as defined in Section~\ref{sec:def}. Let $\mathcal{H}\subset [K]^{\mathcal{X}}$ be a policy set. We follow the convention as in~\cite{syrgkanis2016improved,banihashem2023an} by assuming that $\mathcal{H}$ is finite. Following the same steps as in the online case, we first consider the scenario with side-information as in Section~\ref{sec:side}. Let $(\x_1,c_1),\cdots,(\x_M,c_M)$ be any realization of the feature-loss pairs and $\x_{-N+1}^0$ be the side-information with $\x_{-N+1}^M$ sampled $i.i.d.$ from an (unknown) distribution $\mu$. We consider the following \emph{surrogate} relaxation:
    \begin{equation*}
    \scalebox{0.94}{% 
    $\begin{aligned}
    R_j=\mathbb{E}_{\tilde{\x},\epsilon,Z}\left[-\inf_{h\in \mathcal{H}}\left(2\epsilon_{j+1}[h(\tilde{\x}_{j+1})]Z_{j+1}+\sum_{i=1}^j\hat{c}_i[h(\x_i)]+\sum_{i={j+2}}^M2\epsilon_i[h(\tilde{\x}_i)]Z_i\right)+\gamma (M-j)K\right],
    \end{aligned}$}
    \end{equation*}
where $\epsilon_i$s are $i.i.d.$ uniform over $\{\pm 1\}^K$, $\tilde{\x}_i$s are sampled from $\hat{\mu}_N$ (i.e., the empirical distribution over $\x_{-N+1}^0$) \emph{without replacement} and $Z_i$s are $i.i.d.$ with $Z_i\in \{0,\frac{1}{\gamma}\}$ such that $\mathrm{Pr}[Z_i=\frac{1}{\gamma}]=\gamma K$ and $\gamma$ is a parameter that needs to be tuned. Moreover, $\hat{c}_i$ is a random vector constructed from the prediction $\hat{y}_i\in [K]$ and distribution $q_i$ (see Section~\ref{sec:def}) as $\hat{c}_i=\frac{1}{\gamma} I_i e_{\hat{y}_i}$, where $e_k$ is the standard base of $\mathbb{R}^K$ with coordinate $k$ being $1$ and $I_i$ is the indicator that takes value $1$ w.p. $\frac{\gamma c_i[\hat{y}_i]}{q_i[\hat{y}_i]}$. Similarly, we define the variational relaxation as
\begin{equation*}
    \scalebox{0.94}{% 
    $\begin{aligned}
    \tilde{R}_j=\mathbb{E}_{\x\sim \mu}\mathbb{E}_{\tilde{\x},\epsilon,Z}\left[-\inf_{h\in \mathcal{H}}\left(2\epsilon_{j+1}[h(\x)]Z_{j+1}+\sum_{i=1}^j\hat{c}_i[h(\x_i)]+\sum_{i={j+2}}^M2\epsilon_i[h(\tilde{\x}_i)]Z_i\right)+\gamma (M-j)K\right].
    \end{aligned}$}
    \end{equation*}
\paragraph{Prediction rule.} The prediction rule at step $j$ is given as follows:
\begin{itemize}
    \item[1.] Sample the data $\tilde{\x}$, $\epsilon$ and $Z$ as in the definition of $R_j$;
    \item[2.] Let $\mathcal{D}_K$ be the class of distribution over $[K]$. Find
    \begin{equation*}
    \scalebox{0.95}{% 
    $\begin{aligned}
    \hat{q}_j=\mathop{\mathrm{arg\,min}}\limits_{q\in \mathcal{D}_K}\sup_{p_j\in \mathcal{D}'}\mathbb{E}_{\hat{c}_j\sim p_j}\left\{\langle q,\hat{c}_j\rangle-\inf_{h\in \mathcal{H}}\left(\sum_{i=1}^j\hat{c}_i[h(\x_i)]+\sum_{i={j+1}}^M2\epsilon_i[h(\tilde{\x}_i)]Z_i\right)+\gamma (M-j)K\right\},
    \end{aligned}$}
    \end{equation*}
    where $\mathcal{D}'$ is the class of distributions over $\{\frac{1}{\gamma} e_k:k\in [K]\}\cup \{\mathbf{0}\}$ s.t. $\forall k\in [K], p[k]\le \gamma$ and $\langle q,\hat{c}_j\rangle$ is the scalar product;
    \item[3.] Define $q_j=(1-\gamma K)\hat{q}_j+\gamma \mathbf{1}$ and make prediction $\hat{y}_j\sim q_j$.
\end{itemize}
Note that, the prediction rule is iterative in the sense that the predictions $q_j$s and $\hat{y}_j$s are used to construct $\hat{c}_j$s for the predictors in the following steps.

Let $\Phi$ be the prediction rule as define above.  The bandit minimax regret with side-information is defined as follows:
\begin{equation}
\label{eq:banditside}
\tilde{r}_{M,N}^{\mathsf{bandit}}(\mathcal{H},\Phi)=\mathbb{E}_{\x_{-N+1}^0}\mathbb{E}_{\x_1}\sup_{c_1}\mathbb{E}_{\hat{y}_1}\cdots \mathbb{E}_{\x_M}\sup_{c_M}\mathbb{E}_{\hat{y}_M}\left[\sum_{j=1}^M\langle q_j,c_j\rangle-\inf_{h\in \mathcal{H}}\sum_{j=1}^Mc_j[h(\x_j)]\right],    
\end{equation}
where $\mathbb{E}_{\hat{y}_j}$ is over all the internal randomness used to construct the predictor, including $\tilde{\x}$, $\epsilon$, $Z$, $I_j$ and the randomness of $\hat{y}_j\sim q_j$. It is important to note that the risk of the learner are evaluated on the \emph{expected} loss $\langle q_j,c_j\rangle$ not the point-wise loss $c_j[\hat{y}_j]$, which is different from the online case.

We have the following \emph{approx-admissibility}.
\begin{lemma}
\label{lem:baditadmi}
    For the predictors $q_j$s and $\hat{y}_j$s, we have
    $$\mathbb{E}_{\x_j}\sup_{c_j}\mathbb{E}_{\hat{y}_j}[c_j[\hat{y}_j]+R_j]\le \tilde{R}_{j-1}.$$
\end{lemma}
\begin{proof}[Sketch]
    The proof essentially follows the same arguments as in Lemma~\ref{lem:admiss} and the admissibility proof of~\citet[Theorem 3]{syrgkanis2016improved} by noticing that the only difference between our $R_j$ and the relaxation in ~\cite{syrgkanis2016improved} is the randomness of $\tilde{\x}$s and their entire proof are performed by \emph{conditioning} on such randomness (which only enters in the final step).
\end{proof}

Similarly, we have the following decomposition as presented in  Lemma~\ref{lem:main2}.

\begin{lemma}
\label{lem:bandit2disc}
    For the predictor $\Phi$, the bandit minimax regret with side-information is upper bounded by
    $$\tilde{r}_{M,N}^{\mathsf{bandit}}(\mathcal{H},\Phi)\le \mathbb{E}_{\x_{-N+1}^0}\left[\tilde{R}_0+\sum_{j=1}^{M-1}\mathbb{E}_{\x^j,\hat{y}^j}\sup_{c^j}(\tilde{R}_j-R_j)\right].$$
\end{lemma}
\begin{proof}
    We first observe that $\hat{c}_j$ is a unbiased estimation of $c_j$ for all $j\in [M]$, i.e., $\mathbb{E}_{\hat{y}_j}[\hat{c}_j[k]]=c_j[k]$ for all $k\in [K]$. Therefore,
    \begin{align*}
        \sum_{j=1}^M\langle q_j,c_j\rangle-\inf_{h\in \mathcal{H}}\sum_{j=1}^Mc_j[h(\x_j)]&=\sup_{h\in \mathcal{H}} \sum_{j=1}^M\langle q_j,c_j\rangle-c_j[h(\x_j)]\\
        &=\sup_{h\in \mathcal{H}} \sum_{j=1}^M\langle q_j,c_j\rangle-\mathbb{E}_{\hat{y}_j}[\hat{c}_j[h(\x_j)]]\\
        &=\sup_{h\in \mathcal{H}}\mathbb{E}_{\hat{y}^M}\left[\sum_{j=1}^M\langle q_j,c_j\rangle-\hat{c}_j[h(\x_j)]\right]\\
        &\le \mathbb{E}_{\hat{y}^M}\sup_{h\in \mathcal{H}}\left[\sum_{j=1}^M\langle q_j,c_j\rangle-\hat{c}_j[h(\x_j)]\right]\\
        &=\mathbb{E}_{\hat{y}^M}\left[\sum_{j=1}^M\langle q_j,c_j\rangle-\inf_{h\in \mathcal{H}}\sum_{j=1}^M\hat{c}_j[h(\x_j)]\right].
    \end{align*}
    This implies that
    \begin{align*}
        \tilde{r}_{M,N}^{\mathsf{bandit}}(\mathcal{H},\Phi)&\le \mathbb{E}_{\x_{-N+1}^0}\mathbb{E}_{\x_1}\sup_{c_1}\mathbb{E}_{\hat{y}_1}\cdots \mathbb{E}_{\x_M}\sup_{c_M}\mathbb{E}_{\hat{y}_M}\mathbb{E}_{\hat{y}^M}\left[\sum_{j=1}^M\langle q_j,c_j\rangle+R_M\right]\\
        &\le  \mathbb{E}_{\x_{-N+1}^0}\mathbb{E}_{\x_1}\sup_{c_1}\mathbb{E}_{\hat{y}_1}\cdots \mathbb{E}_{\x_M}\sup_{c_M}\mathbb{E}_{\hat{y}_M}\left[\sum_{j=1}^M\langle q_j,c_j\rangle+R_M\right],
    \end{align*}
    where the second inequality follows by $\sup\mathbb{E}\le \mathbb{E}\sup$.
    The lemma then follows from Lemma~\ref{lem:baditadmi} and the same argument as in the proof of Lemma~\ref{lem:main2}.
\end{proof}

In order to analyze the discrepancy between $R_j$ and $\tilde{R}_j$, we define the following function:
$$f_{\z_j,\hat{c}^j}(\x)=\mathbb{E}_{Z_{j+1}}\left[\inf_{h\in \mathcal{H}}\left(2\epsilon_{j+1}[h(\x)]Z_{j+1}+\sum_{i=1}^j\hat{c}_i[h(\x_i)]+\sum_{i={j+2}}^M2\epsilon_i[h(\tilde{\x}_i)]Z_i\right)\right],$$
where $\z_j$ is composed  of all other variables that define $R_j$ except $Z_{j+1}$ and $\hat{c}^j$.

The following key property bounds the sensitivity of $f_{\z_j,\hat{c}^j}$:
\begin{lemma}
\label{lem:banditsensi}
    We have for any $\z_j,\hat{c}^j$ and $\x,\x'\in \mathcal{X}$
    $$|f_{\z_j,\hat{c}^j}(\x)-f_{\z_j,\hat{c}^j}(\x')|\le 4K.$$
\end{lemma}
\begin{proof}
    Clearly, if $Z_{j+1}=0$ then the sensitivity is $0$, else, $Z_{j+1}=\frac{1}{\gamma}$ and the sensitivity is upper bounded by $\frac{4}{\gamma}$ using the same argument in the proof of Proposition~\ref{prop1}. Since $Z_{j+1}\not=0$ happens w.p. $\le \gamma K$,  the expected function has sensitivity upper bounded by $\frac{4}{\gamma}\times \gamma K=4K$.
\end{proof}

We now observe that 
$$\mathbb{E}_{\x^j,\hat{y}^j}\sup_{c^j}(\tilde{R}_j-R_j)\le \mathbb{E}_{\x^j}\sup_{\hat{c}^j}(\tilde{R}_j-R_j),$$
since $R_j$ and $\tilde{R}_j$ depend only on $\hat{c}_j$. Notably, the support set size of $\hat{c}_j$s is \emph{finite} and upper bounded by $K+1$. Invoking Lemma~\ref{lem:dis2rad}, Lemma~\ref{lem:banditsensi} and a simple application of Massart's lemma gives
$$\mathbb{E}_{\x_{-N+1}^j,\hat{y}^j}\sup_{c_j}(\tilde{R}_j-R_j)\le O\left(\sqrt{\frac{K^2j\log K}{N}}\right).$$
Note that this upper bound is \emph{independent} of the parameter $\gamma$. By Lemma~\ref{lem:bandit2disc} and~\cite[Theorem 3]{syrgkanis2016improved}, we have
$$\tilde{r}_{M,N}^{\mathsf{bandit}}(\mathcal{H},\Phi)\le O\left((KM)^{\frac{2}{3}}(\log|\mathcal{H}|)^{\frac{1}{3}}+K\sqrt{\frac{M^3\log K}{N}}\right).$$
Setting $M(n)=n^{\frac{3}{2}}$ and using a similar argument as Lemma~\ref{lem:side2univer}, we arrive at our main result of this appendix
$$\tilde{r}_T^{\mathsf{bandit}}(\mathcal{H},\Phi)\le O\left((K^{\frac{2}{3}}(\log|\mathcal{H}|)^{\frac{1}{3}}+K\sqrt{\log K})\cdot T^{\frac{4}{5}}\right),$$
where $\Phi$ can be computed \emph{efficiently} by accessing to an ERM oracle (with computational cost same as~\cite{syrgkanis2016improved}).

\section{Oblivious Adversaries}
\label{app:obli}
In this section, we provide the regret analysis for online learning against an \emph{oblivious} adversary as introduced in Section~\ref{sec:def}. We follow the same online learning game as in~(\ref{eq:hbreg}) with the exception that the adversary fixes functions $f_1,\cdots,f_T:\mathcal{X}\rightarrow [0,1]$ before the game and sets the adversary labels $y_t=f_t(\x_t)$ for each time step $t\in [T]$. Formally, for any expert class $\mathcal{H}$ and prediction rule $\Phi$, we are interested in the following \emph{oblivious} minimax regret:
$$\tilde{r}_{T}^{\mathsf{ob}}(\mathcal{H},\Phi)=\sup_{f_1,\cdots,f_T\in [0,1]^{\mathcal{X}}}\sup_{\mu}\mathbb{E}_{\x^T}\mathbb{E}_{\hat{y}^T}\left[\sum_{t=1}^T\ell(\hat{y}_t,f_t(\x_t))-\inf_{h\in \mathcal{H}}\sum_{t=1}^T\ell(h(\x_t),f_t(\x_t))\right],$$
where $\x^T$ are sampled $i.i.d.$ from $\mu$ and $\hat{y}_t\sim \Phi(\x^t,y^{t-1})$ for $t\in [T]$. For the clarity of presentation, we assume that $\ell(\hat{y},y)=|\hat{y}-y|$ is the absolute loss. We now ready to state the main result of this appendix:
\begin{theorem}
\label{thm:obregret}
    Let $\mathcal{H}\subset [0,1]^{\mathcal{X}}$ be a class of Rademacher complexity $\mathsf{Rad}_T(\mathcal{H})=O(T^q)$ for some $q\in [\frac{1}{2},1]$ and $\ell$ be the absolute loss. Then there exists an oracle-efficient prediction rule $\Phi$ with at most $O(\sqrt{T}\log T)$ calls to the ERM oracle per round, such that $\tilde{r}_{T}^{\mathsf{ob}}(\mathcal{H},\Phi)\le O(T^q)$. In particular, for finite-VC class $\mathcal{H}$, we have $\tilde{r}_{T}^{\mathsf{ob}}(\mathcal{H},\Phi)\le O(\sqrt{\mathsf{VC}(\mathcal{H})T})$. For a class $\mathcal{H}$ with $\alpha$-fat shattering dimension $O(\alpha^{-p})$ for some $p>0$, we have $\tilde{r}_{T}^{\mathsf{ob}}(\mathcal{H},\Phi)\le \tilde{O}(T^{\max\{\frac{1}{2},\frac{p-1}{p}\}})$.
\end{theorem}

\begin{proof}
    We will follow the same path as the regret analysis for the \emph{non-oblivious} adversaries as established in Section~\ref{sec:onlinebound}. We first consider the scenario with side-information $\x_{-N+1}^0$, and define for any predictor $\Phi$ the following oblivious minimax regret with side-information:
    $$\tilde{r}_{M,N}^{\mathsf{ob,side}}(\mathcal{H},\Phi)=\sup_{f_1,\cdots,f_M\in [0,1]^{\mathcal{X}}}\sup_{\mu}\mathbb{E}_{\x_{-N+1}^M}\mathbb{E}_{\hat{y}^M}\left[\sum_{j=1}^M\ell(\hat{y}_j,f_j(\x_j))-\inf_{h\in \mathcal{H}}\sum_{j=1}^M\ell(h(\x_j),f_j(\x_j))\right],$$
    where $\x_{-N+1}^M$ are sampled $i.i.d.$ from $\mu$.
    Let $\Phi$ be the predictor as in (\ref{eq:predictor}) and $R_j$ and $\tilde{R}_j$ be the same \emph{surrogate} relaxations as in (\ref{eq:relax}) and (\ref{eq:relaxshit}). We claim that:
    \begin{equation}
        \label{eq:obdecom}
        \tilde{r}_{M,N}^{\mathsf{ob,side}}(\mathcal{H},\Phi)\le \sup_{f^M}\sup_{\mu}\mathbb{E}_{\x_{-N+1}^0}\left[\tilde{R}_0+\sum_{j=1}^{M-1}\mathbb{E}_{\x^j}[\tilde{R}_j-R_j]\right].
    \end{equation}
    To see this, we find
    \begin{align*}
        \tilde{r}_{M,N}^{\mathsf{ob,side}}(\mathcal{H},\Phi)&=\sup_{f^M}\sup_{\mu}\mathbb{E}_{\x_{-N+1}^M}\mathbb{E}_{\hat{y}^M}\left[\sum_{j=1}^M\ell(\hat{y}_j,f_j(\x_j))-\inf_{h\in \mathcal{H}}\sum_{j=1}^M\ell(h(\x_j),f_j(\x_j))\right]\\
        &=\sup_{f^M}\sup_{\mu}\mathbb{E}_{\x_{-N+1}^M}\mathbb{E}_{\hat{y}^M}\left[\sum_{j=1}^M\ell(\hat{y}_j,f_j(\x_j))+R_M\right]\\
        &=\sup_{f^M}\sup_{\mu}\mathbb{E}_{\x_{-N+1}^{M-1}}\mathbb{E}_{\hat{y}^{M-1}}\left[\sum_{j=1}^{M-1}\ell(\hat{y}_j,f_j(\x_j))+\mathbb{E}_{\x_M}\mathbb{E}_{\hat{y}_M}[\ell(\hat{y}_M,f_M(\x_M))+R_M]\right]\\
        &\overset{(a)}{\le} \sup_{f^M}\sup_{\mu}\mathbb{E}_{\x_{-N+1}^{M-1}}\mathbb{E}_{\hat{y}^{M-1}}\left[\sum_{j=1}^{M-1}\ell(\hat{y}_j,f_j(\x_j))+\mathbb{E}_{\x_M}\sup_{y_M}\mathbb{E}_{\hat{y}_M}[\ell(\hat{y}_M,y_M)+R_M]\right]\\
        &\overset{(b)}{\le} \sup_{f^M}\sup_{\mu}\mathbb{E}_{\x_{-N+1}^{M-1}}\mathbb{E}_{\hat{y}^{M-1}}\left[\sum_{j=1}^{M-1}\ell(\hat{y}_j,f_j(\x_j))+\tilde{R}_{M-1}\right]\\
        &= \sup_{f^M}\sup_{\mu}\mathbb{E}_{\x_{-N+1}^{M-1}}\mathbb{E}_{\hat{y}^{M-1}}\left[\sum_{j=1}^{M-1}\ell(\hat{y}_j,f_j(\x_j))+R_{M-1}+\tilde{R}_{M-1}-R_{M-1}\right]\\
        &=\sup_{f^M}\sup_{\mu}\left(\mathbb{E}_{\x_{-N+1}^{M-1}}\mathbb{E}_{\hat{y}^{M-1}}\left[\sum_{j=1}^{M-1}\ell(\hat{y}_j,f_j(\x_j))+R_{M-1}\right]+\mathbb{E}_{\x_{N+1}^{M-1}}(\tilde{R}_{M-1}-R_{M-1})\right)\\
        &\overset{(c)}{\le}\sup_{f^M}\sup_{\mu}\mathbb{E}_{\x_{-N+1}^0}\left[\tilde{R}_0+\sum_{j=1}^{M-1}\mathbb{E}_{\x^j}[\tilde{R}_j-R_j]\right]
    \end{align*}
    where $(a)$ follows by that replacing $f_M(\x_M)$ with $\sup_{y_M}$ do not decrease the value; $(b)$ follows by Lemma~\ref{lem:admiss}; $(c)$ follows by repeating the same argument for another $M-1$ steps.

    Now, the key observation is that $\mathbb{E}_{\x_{-N+1}^j}[\tilde{R}_j-R_j]=0$ for all $j\in [M-1]$ whenever $N\ge M-1$. This follows by the same argument as in the proof of Lemma~\ref{lem:dis2rad} by noticing that the $\sup_{y^j}$ is outside the expectation $\mathbb{E}_{\epsilon'^B}$ for oblivious adversaries. Moreover, this argument holds for all $B=N-M+j+1\ge 1$, i.e., $N\ge M-j$ (since by our assumption $N\ge M-1$ and $j\ge 1$). Therefore, we have
    $$\tilde{r}_{M,N}^{\mathsf{ob,side}}(\mathcal{H},\Phi)\le \mathbb{E}_{\x_{-N+1}^0}[\tilde{R}_0]\le \mathsf{Rad}_M(\mathcal{H})\le O(M^q),$$
    whenever $N\ge M-1$. By the epoch approach as in Section~\ref{sec:epoch} and taking the epoch length $M(n)=2^n$ (which ensures $S(n)\ge M(n)-1$) we conclude
    $$\tilde{r}_T^{\mathsf{ob}}(\mathcal{H},\Psi)\le \sum_{n=1}^{\lceil\log T\rceil} 2^{nq}\le O(T^q),$$
    where $\Psi$ is the epoch predictor derived from $\Phi$ as (\ref{eq:eppredictor}). The theorem now follows by Lemma~\ref{lem:comput} and noticing that the computational error only contributes $O(\sqrt{T})$ to the regret.
\end{proof}

\begin{remark}
    Theorem~\ref{thm:obregret} demonstrates that the \emph{oblivious} minimax regret with \emph{unknown} $i.i.d.$ feature generation process is equivalent to the regret achievable with \emph{known} feature generation distribution and non-oblivious adversaries~\cite[Thm 7]{block2022smoothed}, which also matches the information-theoretical lower bound (upto poly-logarithmic factors).
\end{remark}

\begin{remark}
    It is interesting to note that the proof of Theorem~\ref{thm:obregret} can also be applied to the \emph{semi-adaptive} adversaries that selects the adversary label $y_j$ depending on $\x_{j-B}^j$ for some $B\ge 0$. This provides a way to interpolate all the ranges of regret from $\tilde{O}(T^{\max\{\frac{1}{2},\frac{p-1}{p}\})}$ to $\tilde{O}(T^{\max\{\frac{3}{4},\frac{p+1}{p+2}\}})$.
\end{remark}

\end{document}